\documentclass[letterpaper, 10 pt, conference]{ieeeconf} 
\IEEEoverridecommandlockouts 
\usepackage{graphicx}
\usepackage{subfigure}
\usepackage{epsfig}
\usepackage{cite}
\usepackage{amsmath,amssymb}
\usepackage{algorithmic}
\usepackage{url}
\usepackage{color}
\usepackage[ruled]{algorithm2e}
\usepackage{ntheorem}
\usepackage{float}

\newtheorem{thm}{Theorem}

\newtheorem{rem}{Remark}
\newtheorem{problem}{Problem}

\title{\LARGE \bf
A Lower Bound on Observability for Target Tracking with Range Sensors and its Application to Sensor Assignment}

\author{Lifeng Zhou and Pratap Tokekar% <-this % stops a space
\thanks{The authors are with the Department of Electrical and Computer Engineering, Virginia Tech
  USA. \texttt{\small \{lfzhou, tokekar\}@vt.edu}.}%
\thanks{This material is based upon work supported by the National Science
  Foundation under Grant Nos. 1566247 and 1637915.}}%

\begin{document}

\maketitle
\thispagestyle{empty}
\pagestyle{empty}

%%%%%%%%%%%%%%%%%%%%%%%%%%%%%%%%%%%%%%%%%%%%%%%%%%%%%%%%%%%%%%%%%%%%%%%%%%%%%%%%
\begin{abstract}
We study two sensor assignment problems for multi-target tracking with the goal of improving the observability of the underlying estimator. In the restricted version of the problem, we focus on assigning unique pairs of sensors to each target. We present a $1/3$--approximation algorithm for this problem. We use the inverse of the condition number as the value function. If the target's motion model is not known, the inverse cannot be computed exactly. Instead, we present a lower bound for range-only sensing.

In the general version, the sensors must form teams to track individual targets. We do not force any specific constraints on the size of each team, instead assume that the value function is monotonically increasing and is submodular. A greedy algorithm that yields a $1/2$--approximation. However, we show that the inverse of the condition number is neither monotone nor submodular. Instead, we present other measures that are monotone and submodular. In addition to theoretical results, we evaluate our results empirically through simulations.
\end{abstract}

%%%%%%%%%%%%%%%%%%%%%%%%%%%%%%%%%%%%%%%%%%%%%%%%%%%%%%%%%%%%%%%%%%%%%%%%%%%%%%%%
\section{Introduction}
State estimation is a fundamental problem in sensorics and finds many applications such as localization, mapping, and target tracking~\cite{montemerlo2002fastslam,durrant2006simultaneous}. The estimator performance can be improved by exploiting the \emph{observability} of the underlying system~\cite{gadre2004toward,papadopoulos2010cooperative,arrichiello2013observability,williams2015observability}. We study selecting sensors to improve the observability in tracking a potentially mobile target. 

Observability is a basic concept in control theory and has been widely applied in robotics. Observability for range-only beacon sensors, in particular, has been closely studied for underwater navigation. Gadre and Stilwell~\cite{gadre2004toward} analyzed the local and global observability~\cite{hermann1977nonlinear} for the localization of an Autonomous Underwater Vehicle by an acoustic beacon. The problems of single vehicle localization and multi-vehicle relative localization are studied in \cite{arrichiello2013observability} using an observability criterion introduced in \cite{krener2009measures}. In these works, it is the sensors that are moving. Consequently, the sensors know their control vectors and can thus, compute the observability matrix and its measures. In target tracking with fixed or mobile sensors, however, the control inputs for the targets are unknown. In recent work, Williams and Sukhatme~\cite{williams2015observability} studied a multi-sensor localization and target tracking problem where they showed how to leverage graph rigidity to improve the observability for sensor team localization and robust target tracking. 

While others have studied similar problems in the past~\cite{williams2015observability,arrichiello2013observability}, we focus uniquely on the case when the control inputs for the target are not known to the sensors. Consequently, we cannot compute the observability matrix of the resulting system. Our first contribution is to present a novel lower bound on the observability for the case of unknown target motion tracked by range-only sensors. Specifically, we show how to lower bound the \emph{condition number}~\cite{krener2009measures} of the partially known observability matrix using only the known part (Section~\ref{sec:lowerbound}). 

%We focus on the nonlinear observability analysis for multi-sensor target tracking. We partition the observability matrix into two --- one component representing the sensor-target relative state (which can be controlled by the sensors) and one component representing the target's control inputs (which are unknown and cannot be controlled by the sensors). 

We then study two sensor assignment problems. In the first problem, the goal is to assign unique pairs of sensors to a target. Our second contribution is a greedy assignment algorithm for this problem. We prove that the greedy algorithm achieves a $1/3$--approximation of the optimal solution. We show that the greedy algorithm performs much better than $1/3$ in practice (Section~\ref{sec:simulation}-C). Since the optimal solution cannot be computed, in order to compare the greedy approach we also consider a relaxed version of the assignment problem which gives an upper bound for the optimal.

We then study a general assignment problem where a set of sensors are to be assigned to a specific target. There are no restrictions on the number of targets assigned to a specific sensor. Instead, we let the algorithm decide the optimal configuration of sensor teams assigned to each target. If the weight function is submodular and monotone\footnote{We use ``monotone" and ``monotone increasing" interchangeably.}, a greedy algorithm gives a $1/2$--approximation~\cite{nemhauser1978analysis}. However, our third contribution is to prove that the lower bound of the inverse of the condition number is neither submodular nor monotone. Instead, we use other observability measures such as the trace, log determinant and rank of symmetric observability matrix, and the trace of inverse symmetric observability matrix and show them to be submodular and monotone. We evaluate this algorithm through simulations where we find that sensors are assigned to targets almost uniformly (Section~\ref{sec:simulation}-D). 
% We start with the case of a pair of sensors tracking a mobile target (Section~\ref{sec:simulation}-A). 

% We then show how to compute the worst case configuration maximizing the condition number at each time step. 

% Finally, we apply the lower bound for multi-sensor multi target-tracking where the goal is to assign pairs of sensors to each target. We use the lower bound to solve a maximum weight perfect bipartite matching problem~\cite{cormen2009introduction} in (Section~\ref{sec:selfishmatching}). 

\section{Problem Formulation for sensor assignment} \label{sec:pro_formulation}
We consider a scenario where there are $N$ sensors and $L$ targets in the environment. Our goal is to assign sensors to track the target. We use the notation $\sigma(l)$ to represent the set of sensors assigned to target $t_l$. Similarly, let $\sigma^{-1}(i)$ give the set of targets assigned to sensor $s_i$. We also use $\sigma_i(l)$ to give the $i^{th}$ sensor assigned to $t_l$. We order the assigned sensors by using their IDs such that $\sigma_1(l) < \sigma_2(l) < \sigma_3(l) < \ldots$.  Let $\omega(s_i, s_j, t_l)$, and $\omega(S_l, t_l)$ be some measure of the observability of tracking $t_l$ with $s_i$ and $s_j$, and with a set of sensor(s) $S_l$, respectively. 

We study the following sensor assignment problems. We start with the problem of assigning pairs of sensors to each target.\footnote{Theorems \ref{thm:sensor_n_1} and \ref{thm:sensor_n_gre2} show that at least two sensors are necessary when $\omega(\cdot)$ is the inverse condition number of the observability matrix.} 

\begin{problem}[Unique Pair Assignment] Given a set of sensor positions, $\mathcal{S} := \{s_0,\ldots,s_N\}$ and a set of target estimates at time $t$, $\mathcal{T} := \{t_0,\ldots,t_L\}$, find an assignment of unique pairs of sensors to targets:
\begin{equation}
\text{maximize} \sum_{i=1}^L \omega(\sigma_1(l),\sigma_2(l), t_l)
\end{equation}
with the added constraint each sensor is assigned to at most one target. That is, for all $i = 1,\ldots, N$ we have $|\sigma^{-1}(i)|\leq 1$, assuming $N \geq 2 L$. 
\label{prob:unique}
\end{problem}

We then study the general version of the problem where each target can be tracked by more than two sensors. That is, the sensors form teams of varying sizes to track individual targets. Sensors within a team can share measurements so as to better track the targets. We constrain each sensor to be assigned to only one target, that is, communicate with only one team. This is motivated by scenarios where sensing multiple targets can be time consuming (as is the case with radio sensors~\cite{tokekar2011active}) or communicating multiple measurements can be time and energy consuming.

\begin{problem}[General Assignment] Given a set of sensor positions, $\mathcal{S} := \{s_0,\ldots,s_N\}$ and a set of target estimates at time $t$, $\mathcal{T} := \{t_0,\ldots,t_L\}$, find an assignment of sets of sensors to targets:
\begin{equation}
\text{maximize} \sum_{i=1}^L \omega(\sigma(l), t_l)
\end{equation}
with the added constraint each sensor is assigned to at most one target. 
%That is, for all $i = 1,\ldots, N$ we have $|\sigma^{-1}(i)|\leq 1$. 
%\pcomment{$\sigma(l)\cap \sigma(k)=\emptyset$, if $k\ne l, ~k,l\in{1,2,...,L}$}.
\label{prob:general}
\end{problem}

%Problem~\ref{prob:perfect} can be solved optimally using a standard bipartite matching algorithm. 
We present a $1/3$--approximation to solve Problem~\ref{prob:unique} for any $\omega(\cdot)$. Problem~\ref{prob:general} is the more general assignment problem which is difficult to solve for arbitrary $\omega(\cdot)$. However, for the specific class of submodular functions, there exists a $1/2$--approximation by a greedy algorithm~\cite{nemhauser1978analysis}. Submodularity captures the notion of diminishing returns, i.e., the marginal gain of assigning an additional sensor diminishes as more and more sensors are assigned to track the same target.

A typical measure of observability is the condition number of the observability matrix. When the target is moving, the condition number of the observability matrix cannot be computed, since the control input for the target is unknown to the sensors. We find a lower bound on the inverse of the condition number of the observability matrix. We use this lower bound as our measure, i.e., $\omega(\cdot)$, to find the assignments in Problem \ref{prob:unique}. However, we show that this measure is not submodular (in fact, not even monotone increasing). Instead, we can use a number of other measures (e.g., the trace of symmetric observability matrix, the trace of inverse symmetric observability matrix, log determinant and rank of the symmetric observability matrix) which we know to be submodular and monotone increasing (Theorem~\ref{thm:submodular_observability_measure}). We start by how to bound the inverse of condition number. 
%%%%%%%%%%%%%%%%%%%%%%%%%%%%%%%%%%%%%%%%%%%%%%%%%%%%%%%%%%%%%%%%%%%%%%%%%%%%%%%
\section{Bounding the Observability} \label{sec:lowerbound}
Consider a mobile target whose position is denoted by $o$. Suppose there are $N$ stationary sensors that can measure the distance\footnote{We use the square of the distance/range for mathematical convenience.} to the target. We have:
\begin{equation}
\left\{
                \begin{array}{ll}
                  \dot{o}=u_o,\\
                  z_{i}=h_i(o)=\frac{1}{2}\|p_i-o\|_{2}^{2}, ~i=1,...,N
                \end{array}
              \right.
\label{eqn:multi_sensor_target_system}
\end{equation}
where $o:=[o_x,o_y]^{T}$ gives the 2D position of the target, and $u_o:=[u_{ox},u_{oy}]^{T}$ defines its control input, which is unknown. We assume an upper bound on the control input, given by  $u_{o,\max}=\max\|u_o\|_{2}$. $z_{i}$ defines the range-only measurement from each sensor $s_i$ whose position is given by $p_i=[p_{ix},p_{iy}]^{T}$. For simplicity, we also assume that the target does not collide with any sensor, i.e., $\|p_{i}-o\|_{2} \neq 0$ and no two sensors are deployed at the same position. 

%\noindent \textbf{Notation:} $\|\mathbf{x}\|_{2}: =\sqrt{\mathbf{x}^{T}\mathbf{x}}$ and $\|\mathbf{x}\|_{2}^{2}: =\mathbf{x}^{T}\mathbf{x}$ with vector $\mathbf{x}$. $\|\mathbf{A}|_{\mathrm{F}}:=\sqrt{\mathrm{trace}(\mathbf{A}^{T}\mathbf{A})}$ with real matrix $\mathbf{A}$.

We analyze the weak local observability matrix, $O(o,u_o)$, of this multi-sensor target tracking system. We show how to lower bound the inverse of the condition number of $O(o,u_o)$, given by $C^{-1}(O(o,u_o))$, independent of $u_o$. We also show that the lower bound, $\underline{C}^{-1}(O_{i}(o,u_o))$, is tight.

We compute the local nonlinear observability matrix~\cite{hermann1977nonlinear,williams2015observability} for this system (Equation \ref{eqn:multi_sensor_target_system}) as, 
\begin{equation}
O(o,u_o)=\begin{bmatrix}
    \nabla L_{0}^{h_1} \\ 
     \nabla L_{1}^{h_1}\\ 
     \vdots\\
     \nabla L_{0}^{h_2} \\ 
     \nabla L_{1}^{h_2}\\ 
     \vdots\\
     \\
    \vdots\\
    \\
    \nabla L_{0}^{h_N} \\ 
     \nabla L_{1}^{h_N}\\ 
     \vdots
    
\end{bmatrix}=\begin{bmatrix}
    o_{x}-p_{1x},o_{y}-p_{1y} \\ 
     u_{ox},u_{oy}\\ 
     0,0\\
     \vdots\\
     o_{x}-p_{2x},o_{y}-p_{2y} \\ 
     u_{ox},u_{oy}\\ 
     0,0\\
     \vdots\\
     \\
    \vdots\\
    \\
    o_{x}-p_{Nx},o_{y}-p_{Ny} \\ 
     u_{ox},u_{oy}\\ 
     0,0\\
     \vdots
    
\end{bmatrix}.
\label{observability_multi_sensor_target}
\end{equation}
This equation can be rewritten as,
\begin{eqnarray}
O(o,u_o)&=&\begin{bmatrix}
    o_{x}-p_{1x},o_{y}-p_{1y} \\ 
     o_{x}-p_{2x},o_{y}-p_{2y} \\ 
    \vdots\\
    o_{x}-p_{Nx},o_{y}-p_{Ny} \\ 
     u_{ox},u_{oy}   
\end{bmatrix}_{(N+1)\times 2}.
\label{eqn:new_multi_sensor_target}
\end{eqnarray}
The state of the target $o$ is \emph{weakly locally observable} if the local nonlinear observability matrix has full column rank~\cite{hermann1977nonlinear}. However, the rank test for the observability of the system is a binary condition which does not tell the degree of the observability or how \emph{good} the observability is.  The \emph{condition number}~\cite{krener2009measures}, defined as the ratio of the largest singular value to the smallest, can be used to measure this degree of unobservability. A larger condition number suggests worse observability. We use the \emph{inverse of condition number} given as,
\begin{equation}
C^{-1}(O(o,u_o))=\frac{\sigma_{\min}(O(o,u_o))}{\sigma_{\max}(O(o,u_o))}
\label{eqn:inverse_condition_number}.
\end{equation}
Note that, $C^{-1} \in [0,1]$. $C^{-1}=0$ means  $O(o,u_o)$ is singular and $C^{-1}=1$ means $O(o,u_o)$ is \emph{well conditioned}. A larger $C^{-1}$ means better observability (see more details in~\cite{arrichiello2013observability}).

In the local nonlinear observability matrix $O(o,u_o)$, $u_o$ is unknown and not controllable by the sensor. On the other hand, $o-p_{i}$, depends on the relative state between each sensor $s_i$ and target $o$ and is known to the sensor (assuming an estimate of the target's position is known). The system can control $o-p_{i}$ either by moving the sensors or assigning new sensors to track the target. 
\begin{thm}
For the multi-sensor-target system (Equation \ref{eqn:multi_sensor_target_system}) with the number of sensors, $N\geq 2$, the inverse of the condition number is lower bounded by $\dfrac{\sigma_{\min}(O(o))}{\sqrt{\sigma^{2}_{\max}(O(o))+ u_{o,\max}^{2}}}$.
\label{thm:lower_bound}
\end{thm}
We present the full proof for this and all other results in the in the appendix.%accompanying report~\cite{observability_paper_old_version}.

We wish to improve the worst case, i.e., the lower bound of $C^{-1}(O(o,u_o))$, by optimizing the sensor-target relative state, which can be controlled by the sensor. For example, if the sensors are mobile, they can move so as to improve the lower bound. If only a subset of sensors are active at a time, we can choose the appropriate subset to improve the lower bound. In the following, we will show that at least two sensors are required to improve the lower bound.
\begin{thm}
The lower bound of the observability metric in one-sensor-target system, $\underline{C}^{-1}(O_{i}(o,u_o))$, cannot be controlled by the sensor.
\label{thm:sensor_n_1}
\end{thm}
When the number of sensors, $N\geq 2$, we have a positive result that shows that the sensors can improve the lower bound on the condition number of optimizing their positions.
\begin{thm}
Suppose that the number of sensors, $N \geq 2$. Even though the contribution to the observability matrix from the target's input, $O(u_o)$, is unknown and cannot be controlled, if the sensors increase $C^{-1}(O(o))$ and $\sigma_{\min}(O(o))$ (the inverse of condition number and the smallest singular number of the relative state contribution $O(o)$), then the lower bound of $C^{-1}(O(o,u_o))$ also increases. 
\label{thm:sensor_n_gre2}
\end{thm} 
% We present the full proof and Remark \ref{remark:rem2} in the appendix.
\begin{rem}
The lower bound $\underline{C}^{-1}(O(o,u_o))$ is tight when the target is known to be stationary. If $u_o\in\{0\}$, $O(o,u_o)=O(o)$ by Equation~\ref{eqn:new_multi_sensor_target} and Theorem \ref{thm:lower_bound}. Thus, the lower bound  $\underline{C}^{-1}(O(o,u_o))={C}^{-1}(O(o,u_o))$ implying that the lower bound is tight.
\end{rem}

% \subsection{Illustrative Case of Pair of Sensors and One Target} \label{sec:pair}
% In this section, we will focus on the 
Consider the special case when $N=2$. 
% We know that with $N=1$, the sensor cannot optimize the lower bound. We will show how to improve the lower bound for the case of two sensors. 
We use $i$ and $j$ to denote the two sensors tasked with tracking the target. The local observability matrix $O_{i,j}(o,u_o)$ for the $i-j-o$ system is given by,
\begin{equation}
O_{i,j}(o,u_o)=\begin{bmatrix}
    o_{x}-p_{ix} & o_{y}-p_{iy} \\ 
    o_{x}-p_{jx} & o_{y}-p_{jy} \\
     u_{ox} & u_{oy}
\end{bmatrix}.
\label{eqn:observability_i_j_o}
\end{equation}

For ease of notation, we represent the relative sensor-target position and sensor-target orientation in polar coordinates with the target at the center:
\begin{equation}
\left\{
                \begin{array}{ll}
                  p_{ix}-o_{x}=d_{io}\cos\theta_i,\\
                  p_{iy}-o_{y}=d_{io}\sin\theta_i,\\
                  p_{jx}-o_{x}=d_{jo}\cos\theta_j,\\
                  p_{jy}-o_{y}=d_{jo}\sin\theta_j.
                \end{array}
              \right.
\label{eqn:polar_relative_sensor_target}
\end{equation}
where $d_{io}$, $d_{jo}$, $\theta_i$ and $\theta_j$ indicate the relative distances and orientations between sensors $i,j$ and target $o$ as shown in Figure \ref{fig:sensor_target_polar}.

\begin{figure}[htb]
\centering
\includegraphics[width=0.6\columnwidth]{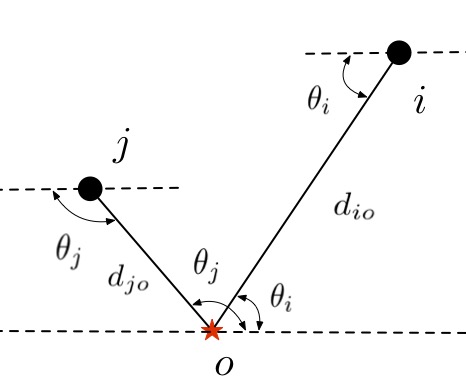}
\caption{Relative sensor-target position and sensor-target orientation in polar coordinates.\label{fig:sensor_target_polar}}
\end{figure}
\begin{thm}
The lower bound of the inverse of condition number of the observability matrix for $i-j-o$ system is given by,
\begin{eqnarray}
&&\underline{C}^{-1}(O_{i,j}(o,u_o)) = \nonumber\\
&&\sqrt{\frac{\lambda_{\min}(O^{T}(o)O(o))}{\lambda_{\max}(O^{T}(o)O(o)) + u_{o}^{2}}}=\nonumber\\
&& \sqrt{\frac{1+\alpha^{2}-\sqrt{1+\alpha^{4}+2\alpha^{2}\cos(2\theta_{ji})}}{1+\alpha^{2}+\sqrt{1+\alpha^{4}+2\alpha^{2}\cos(2\theta_{ji})}+2u_{o}^{2}/d_{io}^{2}}},~
\label{eqn:lower_bound_i_j_o}
\end{eqnarray}
where $\alpha:=d_{jo}/d_{io}$. 
\label{thm:lower_bound_ijo}
\end{thm}

We plot $\underline{C}^{-1}(O_{i,j}(o,u_o))$ as a function of $\alpha \in [0, 5]$ and angle $\theta_{ji} \in [-\pi, \pi]$ with $u_{o,\max}$ and $d_{io}$ selected as $1~m/s$ and $1~m$, respectively (Figure \ref{fig:lower_bound_surf_contour}). Note that $\underline{C}^{-1}(O_{i,j}(o,u_o))$ reaches its maximum when $\alpha=1$ and $\theta_{ji}=\pm \frac{\pi}{2}$. That is, both sensors are at the same distance from the target and are perpendicular with respect to the target. On the other hand, when $\theta_{ji}=0$, $\theta_{ji}=\pi$, $\alpha=0$ or $\alpha \to \infty$, $\underline{C}^{-1}(O_{i,j}(o,u_o))$ reaches zero. We summarize the results in the following theorem.

\begin{figure*}[htb]
\centering{
\subfigure[]{\includegraphics[width=0.48\columnwidth]{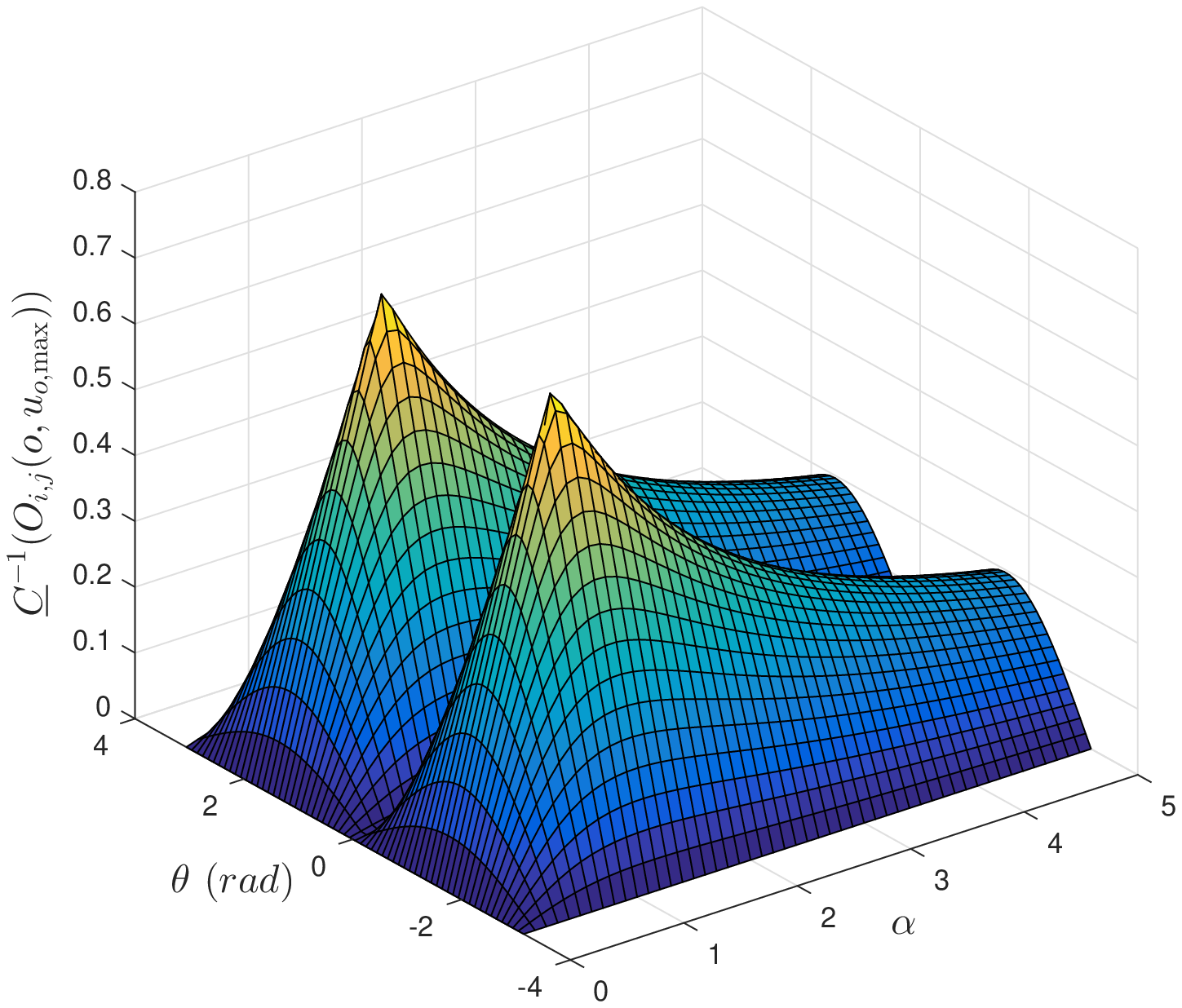}}
\subfigure[]{\includegraphics[width=0.48\columnwidth]{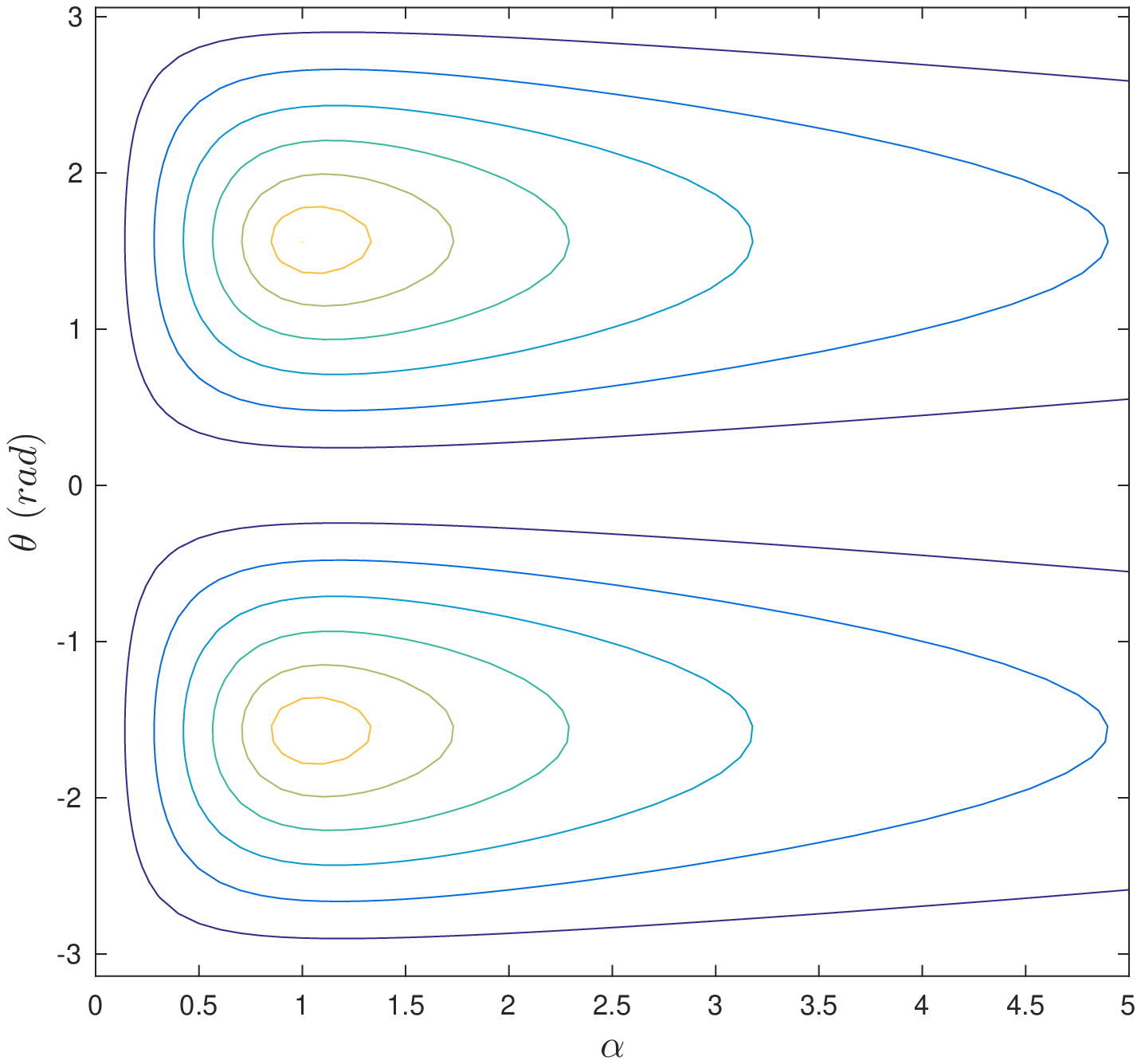}}
\subfigure[]{\includegraphics[width=0.48\columnwidth]{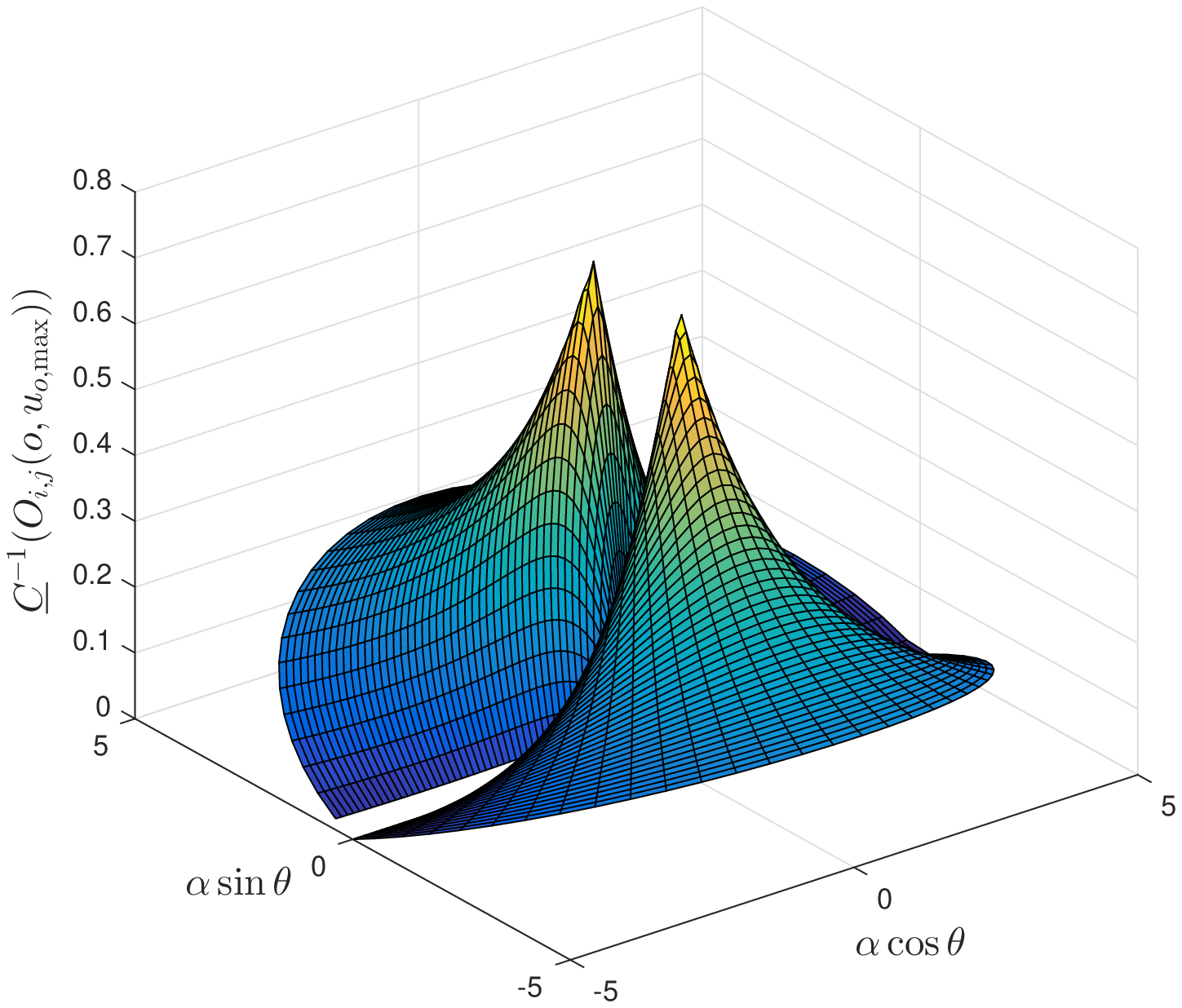}}
\subfigure[]{\includegraphics[width=0.48\columnwidth]{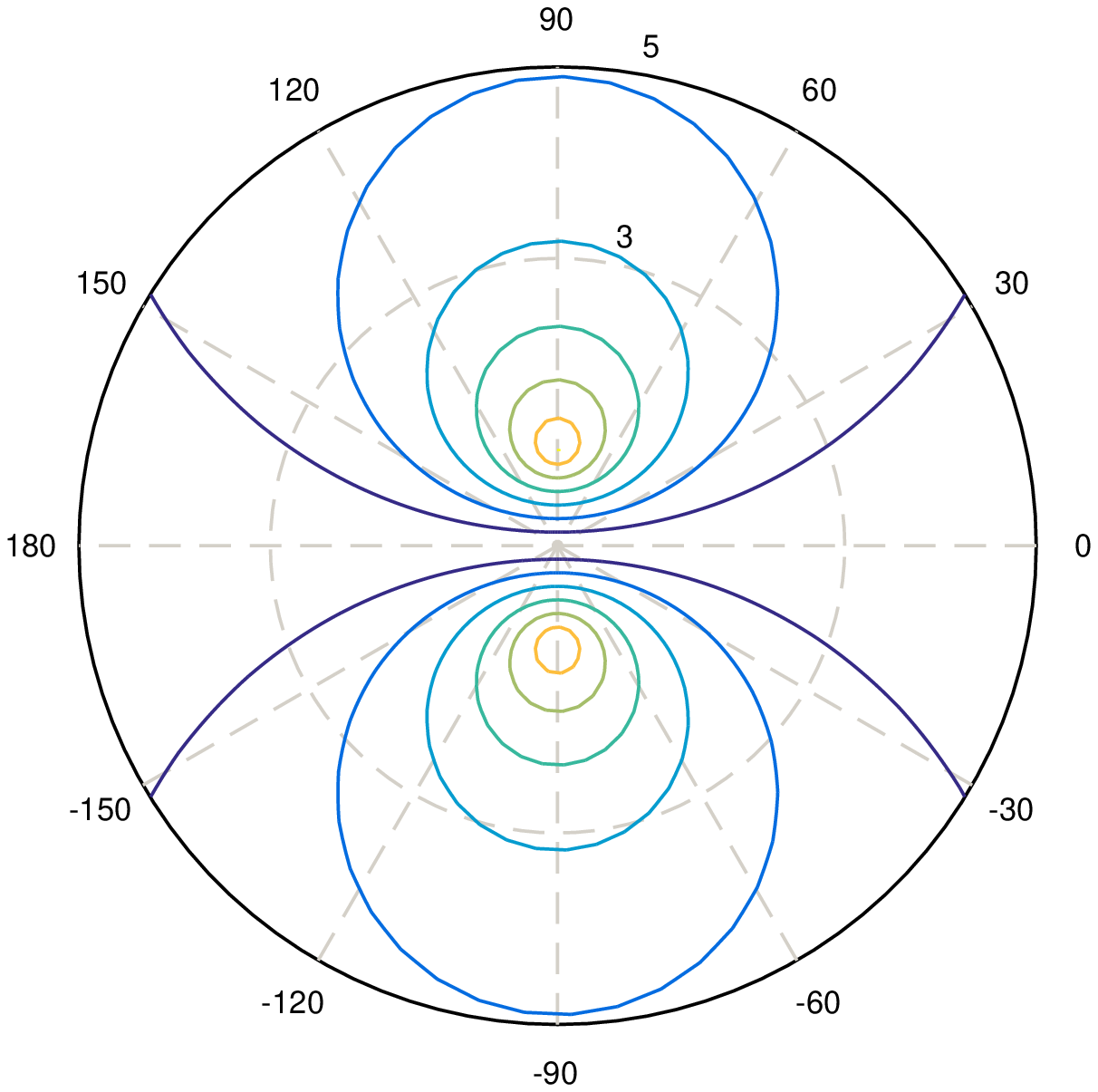}}
}
\caption{The lower bound of the inverse of condition number of the observability matrix for the $i-j-o$ system as a function of the distance ratio, $\alpha$, and angle $\theta$ in Cartesian (a, b) and Polar (c, d) coordinates.\label{fig:lower_bound_surf_contour}} 
\end{figure*}

\begin{thm}
The lower bound of the inverse of the condition number for the observability matrix for $i-j-o$ system, $\underline{C}^{-1}(O_{i,j}(o,u_o))$, reaches its maximum, $\sqrt{\frac{1}{1+u_{o}^{2}/d_{io}^{2}}}$, at $\alpha=1$ and $\theta_{ji}=\pm \frac{\pi}{2}$, and reaches its minimum, zero, at $\theta_{ji}=0$, $\theta_{ji}=\pi$, $\alpha=0$, or $\alpha \to \infty$. For a fixed $\theta_{ji} \neq 0, \pm \frac{\pi}{2}, \pi$, the minima occurs when $\alpha=0$ or $\alpha \to \infty$, but the maximum extreme point is not at $\alpha=1$.
\label{thm:$ij_pair_extreme$}
\end{thm}

Next, we solve two assignment problems for selecting sensor(s) to track  target(s). 
%%%%%%%%%%%%%%%%%%%%%%%%%%%%%%%%%%%%%%%%%%%%%%%%%%%%%%%%%%%%%%%%%%%%%%%%%%%%%%%
\section{Assignment Algorithms}
So far, we have assumed that we know the true position, $o(k)$, of the target at time $k$. In practice, we only have an estimate, $\hat{o}(k)$, for $o$ along with its covariance $\Sigma(k)$. The estimate is obtained by fusing the past measurements obtained by the sensors using, for example, an Extended Kalman Filter (EKF). 
% \pcomment{We first consider the simple case when the target is moving on some unknown trajectory, independent of the sensors. At each timestep, we compute the best pair of sensors to track the target using Algorithm~\ref{algorithm:best_pair}.}

% \begin{algorithm}
% $k\leftarrow 0$\\
% \While{true}{
%  %\KwData{this text}
%  %\KwResult{how to write algorithm with \LaTeX2e }
% $(i,j)\leftarrow$ pair that maximizes $\underline{C}^{-1}(O_{i,j}(\hat{o}(k-1),u_{o,\max}))$\\
% $\{z_i(k),z_j(k)\}\leftarrow$ obtain measurements for $i$ and $j$\\
% $\hat{o}(k), \Sigma(k) \leftarrow$ EKF update with $z_i,z_j$\\
% $k\leftarrow k + 1$
% }   
%  \caption{Best Pair Strategy}
%  \label{algorithm:best_pair}
% \end{algorithm}
Since the sensors know the probability distribution of  target's state with mean $\hat{o}$, and covariance $\Sigma$, they can calculate the relative distances, $d_{io}$, by using the Mahalanobis distance \cite{mahalanobis1936generalized}, 
\begin{equation*}
d_{io}=\sqrt{(\hat{o}-p_{i})^{T}\Sigma^{-1}(\hat{o}-p_{i})}, ~i\in\{1,,...,N\}.
\label{eqn:mahalanobis}
\end{equation*}
% \emph{Perfect Pair Assignment using Bipartite Graph Matching} 
% %\label{sec:perfectmatching}

% Since pair sensor can achieve a good performance (Theorem \ref{thm:$ij_pair_extreme$}), we firstly study the assignment problem where pair sensor is assigned to a target. If each sensor pair can be assigned to a target once, and each target can be measured by only one sensor pair, 

\subsection{A $1/3$--approximation algorithm for Problem~\ref{prob:unique}} \label{sec:selfishmatching}

% Note that, in Maximum Weight Perfect Bipartite Matching Problem, each sensor can be paired multiple times with other sensors. If each sensor can be matched with another sensor at most once and assigned to one specified target, a Selfish Matching Problem is constructed as below:
%  %Here, we consider the case that sensor pair is assigned to one target. Generally, pair sensors work better than single sensor in terms of of observability. However, if the control input of target is known or can be estimated, sometimes, single sensor can also have a good observability performance. 
% \begin{eqnarray*}
% &&\max_{M\subset E} \omega(M)\nonumber\\
% %&& s.t. ~\mathrm{if} ~(s_i,t_{l}) \in M, \\
% %&& ~~~~~\mathrm {then} (s_i',t_l) \notin M ~\&~ (s_i,t_l') \notin M. \nonumber\\
% && ~~~~~\mathrm{if} ~((s_i,s_j),t_{l}) \in M, \nonumber\\
% && ~~~~~\mathrm {then} ((s_i',s_j),t_{l}) \notin M ~\&~ ((s_i,s_j'),t_{l}) \notin M \nonumber\\
% && ~~~~~~\&~ ((s_i,s_j),t'_{l})  \notin M.
% \label{eqn:maximum_matching_problem}
% \end{eqnarray*}
% We first study the Unique Pair Assignment (Problem \ref{prob:unique}) where each sensor $s_i$ can be matched with another sensor $s_j$ at most once and assigned to at most one specified target $t_l$. 
We propose a greedy algorithm to solve Problem~\ref{prob:unique}. In each round, we calculate the observability metric, $\omega(\sigma_1(l),\sigma_2(l), t_l)$ for
%all twotuples $(s_i,t_{l})$ and 
all triples $(\sigma_1(l),\sigma_2(l), t_l), ~\sigma_1(l),\sigma_2(l)\in \mathcal{S},~t_l\in\mathcal{T}$, and select the triple which has the maximum $\omega(\sigma_1(l),\sigma_2(l), t_l)$, then remove 
%$\{s_i,t_l\}$ or 
$\{\sigma_1(l),\sigma_2(l)\}$ from sensor set $\mathcal{S}$ and remove $t_l$ from target set $\mathcal{T}$, respectively. We present the greedy approach in Algorithm \ref{algorithm:unique_pair_assignment} where $\omega(\textrm{GREEDY})$ denotes total value charged by the greedy approach. We can use the inverse of the condition number (Equation~\ref{eqn:lower_bound_i_j_o}) as $\omega(\cdot)$.
% and $\underline{C}^{-1}(O_{\sigma_1(l),\sigma_2(l)}(\hat{o}_{l}(k-1),u_{o_{l},\max}))$ denotes the lower bound of the inverse of the condition number of the observability matrix for system involved with sensor pair $(\sigma_1(l),\sigma_2(l))$ and target $t_l$. Sensor pair can only obtain the estimate position of the target, $\hat{o}_{l}$ by using EKF and use the maximum control input of target, $u_{o_{l},\max}$ to calculate the lower bound of the inverse of condition number, which we used for the measure of the observability in Unique Pair Assignment problem.
\begin{algorithm}
$k\leftarrow 0, ~\omega(\textrm{GREEDY})\leftarrow 0$\\
\While{true}{
Compute all possible 
$\omega(\sigma_1(l),\sigma_2(l), t_l)$.\\
% :=\{\underline{C}^{-1}(O_{\sigma_1(l),\sigma_2(l)}(\hat{o}_{l}(k-1),u_{o_{l},\max}))\}$\\
Select %twotuple $(s_i,t_{l})$ or 
triple $(\sigma_1(l),\sigma_2(l), t_l)$ with maximum $\omega(\sigma_1(l),\sigma_2(l), t_l)$ defined as $\omega_{\max}$.\\ $\omega(\textrm{GREEDY})\leftarrow \omega(\textrm{GREEDY})+\omega_{\max}$.\\
% $\mathcal{S}\backslash\{s_i,s_j\}$ and $\mathcal{T}\backslash\{t_l\}$ $\leftarrow$ 
Remove $\{s_i,s_j\}$ from union sensor $S$ and remove $t_l$ from target set $\mathcal{T}$.\\
$k\leftarrow k + 1$
}   
 \caption{Greedy Unique Pair Assignment}
 \label{algorithm:unique_pair_assignment}
\end{algorithm}
%Here, the observability metric $\{\underline{C}^{-1}(O_{i}(\hat{o}_{l}(k-1),u_{o_{l},\max}))\}$ represents the edge weight $\omega$. We also call our greedy selfish matching algorithm as weighted greedy selfish matching algorithm (GSM), which can be modified to unweighted version by transforming weight $\omega$ to $\omega/1$ copies with unit weight $1$. The problem can be transformed from maximizing total weight to maximizing total copies. 

We present the following lemmas to guarantee the effectiveness of the greedy algorithm. 
% \begin{lem}
% Suppose GREEDY picks triple $(s_i,s_j,t_l)$ in current round. We will receive the reward of $(s_i,s_j,t_l)$ to at most 3 triples in OPT, all of which have reward no more than that has not been gained in a previous round. 
% \label{lem:lem_greedy_optimal}
% \end{lem}
% \begin{proof}
% Given different sensors $s_i, s_j, s_k, s_p, s_q, s_r, s_x$ and different targets $t_l, t_m, t_n$. If GREEDY picks triple $(s_i,s_j,t_l)$ with reward $c_g$, optimal can choose triple $(s_i,s_j,t_l)$ with reward less than or at most equal $c_g$, or choose $(s_i,s_k,t_l)$ and $(s_j,s_p,t_m)$ with total rewards less than or at most equal to $2c_g$ or choose $(s_i,s_p,t_m)$, $(s_j,s_q,t_n)$ and $(s_r,s_x,t_l)$ with total rewards less or at most equal to $3c_g$.
% \end{proof}
% \begin{lem}
% After the last round, all pairs in OPT have been assigned.
% \label{lem:lem_optimal_fin}
% \end{lem}
% \begin{proof}
% Based on the assumption that $N\geq 2L$ in Problem~\ref{prob:unique}, the claim is always true.
% \end{proof}
\begin{thm}
$\omega(\textrm{GREEDY}) \geq \frac{1}{3}\textrm{OPT}$ where \textrm{OPT} is the optimal algorithm for Problem~\ref{prob:unique}.
\label{thm:lem_optimal_fin}
\end{thm}
\begin{proof}
Suppose GREEDY picks triple $(s_i,s_j,t_l)$ in the $k^{th}$ round. We will charge the value $\omega(s_i,s_j,t_l)$ to at most three triples in OPT that have not been charged in previous rounds. Furthermore, if a triple in OPT is charged in the $k^{th}$ round, then we show that the $\omega(\cdot)$ value of the triple is less than $\omega(s_i,s_j,t_l)$.

There are three cases. 
\begin{enumerate}
\item $(s_i,s_j,t_l)$ is also chosen by OPT. In this case, we will charge exactly $\omega(s_i,s_j,t_l)$ to the triple $(s_i,s_j,t_l)$ in OPT, if it has not been charged previously. 
\item Exactly two of $(s_i,s_j,t_l)$ appear in a triple chosen by OPT. Consider the case where OPT chooses $(s_i,s_j,t_m)$ where $m\neq l$. All other cases are symmetric. In this case, we need to charge $\omega(s_i,s_j,t_l)$ to at most two triples --- $(s_i,s_j,t_m)$ and the one containing $t_l$, say $(s_p,s_q,t_l)$. If $(s_i,s_j,t_m)$ has not been charged previously, then it must mean that GREEDY has not assigned any sensors to $t_m$ in previous rounds. Since GREEDY chose $(s_i,s_j,t_l)$ in the $k^{th}$ round, it must mean $\omega(s_i,s_j,t_l) \geq \omega(s_i,s_j,t_m)$, otherwise GREEDY would have chosen $(s_i,s_j,t_m)$ in the $k^{th}$ round. Likewise, if $(s_p,s_q,t_l)$ has not been charged previously, we must have $\omega(s_i,s_j,t_l) \geq \omega(s_p,s_q,t_l)$. Therefore, the value charged in the $k^{th}$ round will be at most twice of $(s_i,s_j,t_l)$.
\item No two of $(s_i,s_j,t_l)$ appear in the same triple chosen by OPT. In this case, we can charge the value of $\omega(s_i,s_j,t_l)$ to at most three triples. Using an argument similar to the previous case, we can say that the value charged in the $k^{th}$ round will be at most thrice of $\omega(s_i,s_j,t_l)$.
\end{enumerate}

Therefore, once all triples in GREEDY are charged, it follows that $\omega(\textrm{GREEDY}) \geq \frac{1}{3}\textrm{OPT}$.
\end{proof}

% Now, it is time to find the optimal solution for the verification of the proposed greedy approach. However\pcomment{strunk \& white}, 

\subsection{General assignment using Submodular Welfare Optimization}\label{sec:general_assign}
% Second, we extend the assignment problem to a scenario where a set of sensors (not just sensor pair) can be assigned to a target (Section \ref{sec:non_submodular_cond}). we find that the inverse of the condition number is neither neither monotone increasing nor submodular (Theorem \ref{Them:invcond_not_sub}), which makes this assignment problem even more challenging.
The sensor assignment problem where only two sensors are assigned is discussed above. We now study a more general assignment (Problem~\ref{prob:general}) where each target $t_l$ is tracked by a subset of sensors $\sigma(l)\subset S, ~l\in\{1,2,...,L\}$ whose cardinality is not necessarily two. 
% We formulate this assignment problem as General Assignment (Problem \ref{prob:general}).
% $S_{t_l}\subset S, ~t_l\in\{1,2,...,L\}$ whose cardinality is not necessarily 2. We formulate the General Sensor Assignment Problem as below: 
% \begin{eqnarray*}
% &&\max\sum_{t_l=1}^{L}\omega(S_{t_l})\\
% && s.t. ~\bigcup_{{t_l}=1}^{L}S_{t_l}=S, \nonumber\\
% && \mathrm{and} ~(S_{t_1},S_{t_2},...,S_{t_L}) \mathrm{~are ~disjoint.}
% \label{eqn:maximum_submodular_problem}
% \end{eqnarray*}
% where $\omega(S_{t_l})$ denotes the lower bound of the inverse of the condition number of the observability matrix for the system where target $t_l$ is measured by a set of sensors, $S_{t_l}$. 
This is known as \emph{submodular welfare problem} in the literature\cite{vondrak2008optimal} where the objective is to maximize $\sum_{i=1}^{n}w_i(S_i)$ for independent sets $\{S_i| S_i\in S, ~i=\{1,2,...n\}\}$ by using monotone and
submodular utility functions $w_i$. A greedy algorithm~\cite{nemhauser1978analysis} yields a $1/2$--approximation for this problem. 
% And Vondr{\'a}k\cite{vondrak2008optimal} has given 
We first show that the lower bound of the inverse of the condition number is neither monotone nor submodular which makes the optimization problem much more challenging.
\begin{thm}
The lower bound of the inverse of condition number function $\omega(\cdot)$ is neither monotone increasing nor submodular. 
\label{Them:invcond_not_sub}
\end{thm}
\begin{proof}
We prove the claim by giving two counter-examples. 

\noindent\textbf{Case 1}: Given the sensors $s_1(0,0)$, $s_2(2\sqrt{3},-9)$, $s_3(\sqrt{3},3)$ and target $t_1(\sqrt{3},1)$ with $u_{o_{1},max}=1$ in 2-D plane (Figure~\ref{fig:position_cond_test_nonsubmodular}-(a)), $\omega(\{s_1,s_3\})=0.5345>\omega(\{s_1,s_2, s_3\})=0.1823$, which shows $\omega(\cdot)$ is not monotone increasing. 

\noindent\textbf{Case 2}: Given the sensors $s_1(0,0)$, $s_2(2\sqrt{3},0)$, $s_3(\sqrt{3},0.1)$, $s_4(\sqrt{3},3)$ and target $t_1(\sqrt{3},1)$ with $u_{o_{1},max}=1$ in 2-D plane (Figure~\ref{fig:position_cond_test_nonsubmodular}-(b)), $\omega(\{s_1,s_2,s_3\})-\omega(\{s_1,s_2\})=0.3310-0.5345=-0.2035<\omega(\{s_1,s_2,s_4,s_3\})-\omega(\{s_1,s_2, s_4\})=0.8765-0.9258=-0.0493$, which shows $\omega(\cdot)$ is not submodular. 
\end{proof}
\begin{figure}[htb]
\centering{
\subfigure[\textbf{Case 1}]{
\includegraphics[width=0.8\columnwidth]{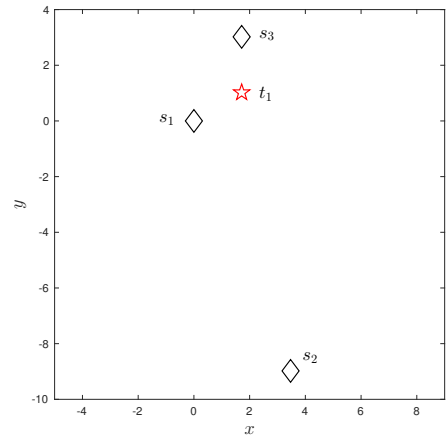}}
\subfigure[\textbf{Case 2}]{
\includegraphics[width=0.8\columnwidth]{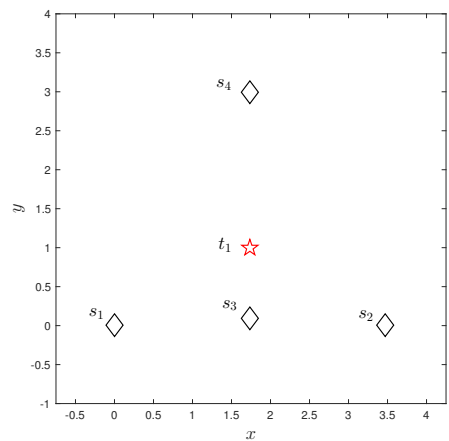}}
}
\caption{Positions of sensors and target in 2D plane.\label{fig:position_cond_test_nonsubmodular}}
\end{figure}

Therefore, we focus on other measures of observability and summarize the results in Theorem~\ref{thm:submodular_observability_measure}.
\begin{thm}
For \emph{symmetric observability matrix},  $\mathbb{O}(o,u_o):=O^{T}(o,u_o)O(o,u_o)$, its trace, log determinant, rank  and the trace of its inverse are submodular and monotone increasing.
\label{thm:submodular_observability_measure}
\end{thm}

The proof is similar to proving that the trace of the Gramian and inverse Gramian, the log determinant, and the rank of the Gramian are monotone submodular\cite{summers2016submodularity}. Resorting to submodular and monotone observability measure, we can use a simpler greedy algorithm to solve the General Assignment Problem. 
\section{Simulations}\label{sec:simulation}
We illustrate the performance of the pairing and assignment strategies for sensor selection using  observability measure as the performance criterion. We first consider the two sensors case for tracking a moving target, and then focus on the multi-sensor multi-target assignment problems. The video\footnote{\url{https://youtu.be/Kt0yeYXyrKQ}} and the code\footnote{\url{https://github.com/lovetuliper/observability-based-sensor-assignment.git}} of our simulations are available online.

\subsection{Two Sensors Case} \label{subsec:pairing_single_target}
%We first consider the simple case when the target is moving on some unknown trajectory, independent of the sensors. At each timestep, we compute the best pair of sensors to track the target using Algorithm~\ref{algorithm:best_pair}.

% \begin{algorithm}
% $k\leftarrow 0$\\
% \While{true}{
%  %\KwData{this text}
%  %\KwResult{how to write algorithm with \LaTeX2e }
% $(i,j)\leftarrow$ pair that maximizes $\underline{C}^{-1}(O_{i,j}(\hat{o}(k-1),u_{o,\max}))$\\
% $\{z_i(k),z_j(k)\}\leftarrow$ obtain measurements for $i$ and $j$\\
% $\hat{o}(k), \Sigma(k) \leftarrow$ EKF update with $z_i,z_j$\\
% $k\leftarrow k + 1$
% }   
%  \caption{Best Pair Strategy}
%  \label{algorithm:best_pair}
% \end{algorithm}
\begin{figure*}[htb]
\centering{
\subfigure[$k=1$]{
\includegraphics[width=0.66\columnwidth]{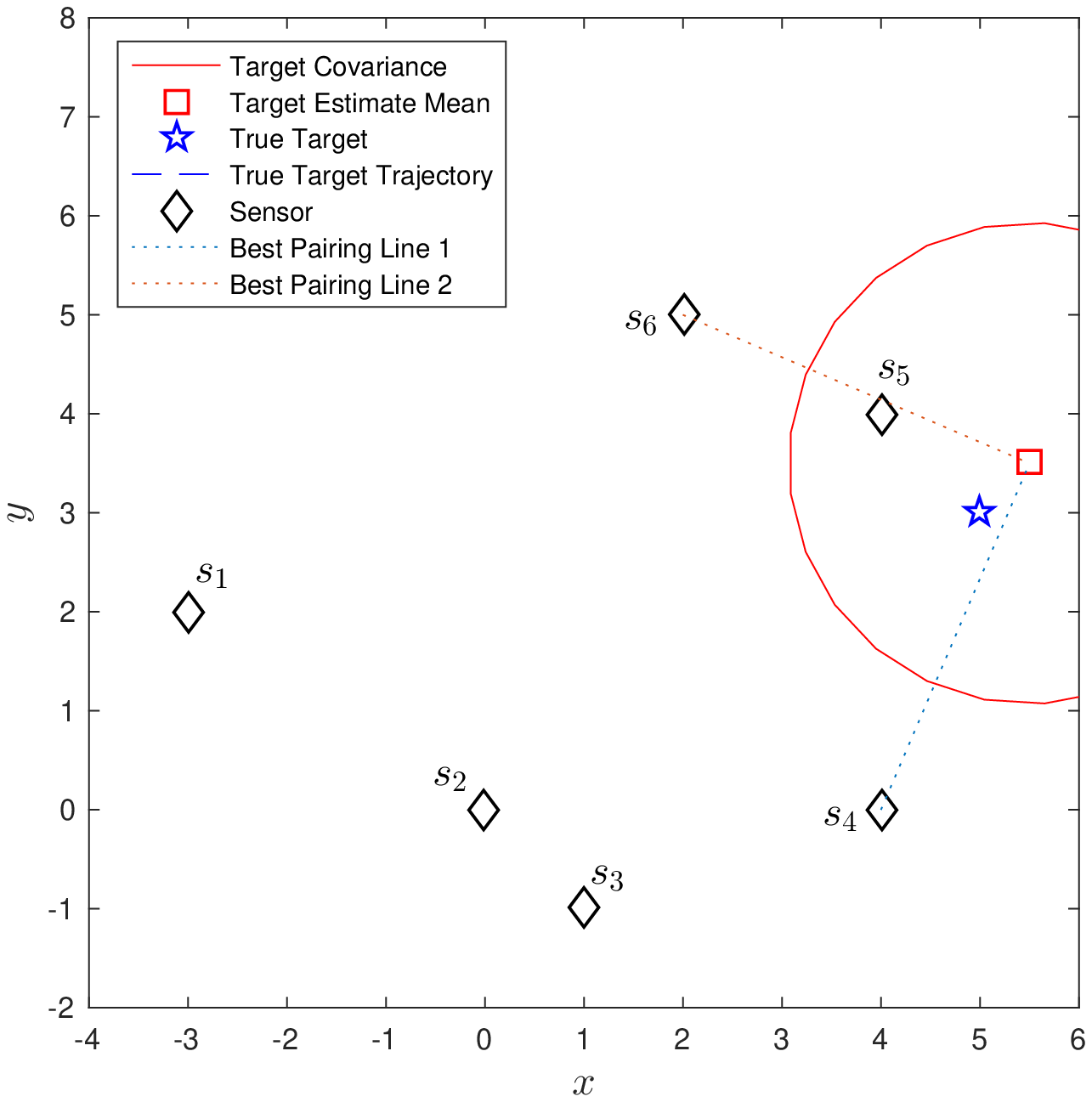}}
\subfigure[$k=50$]{
\includegraphics[width=0.66\columnwidth]{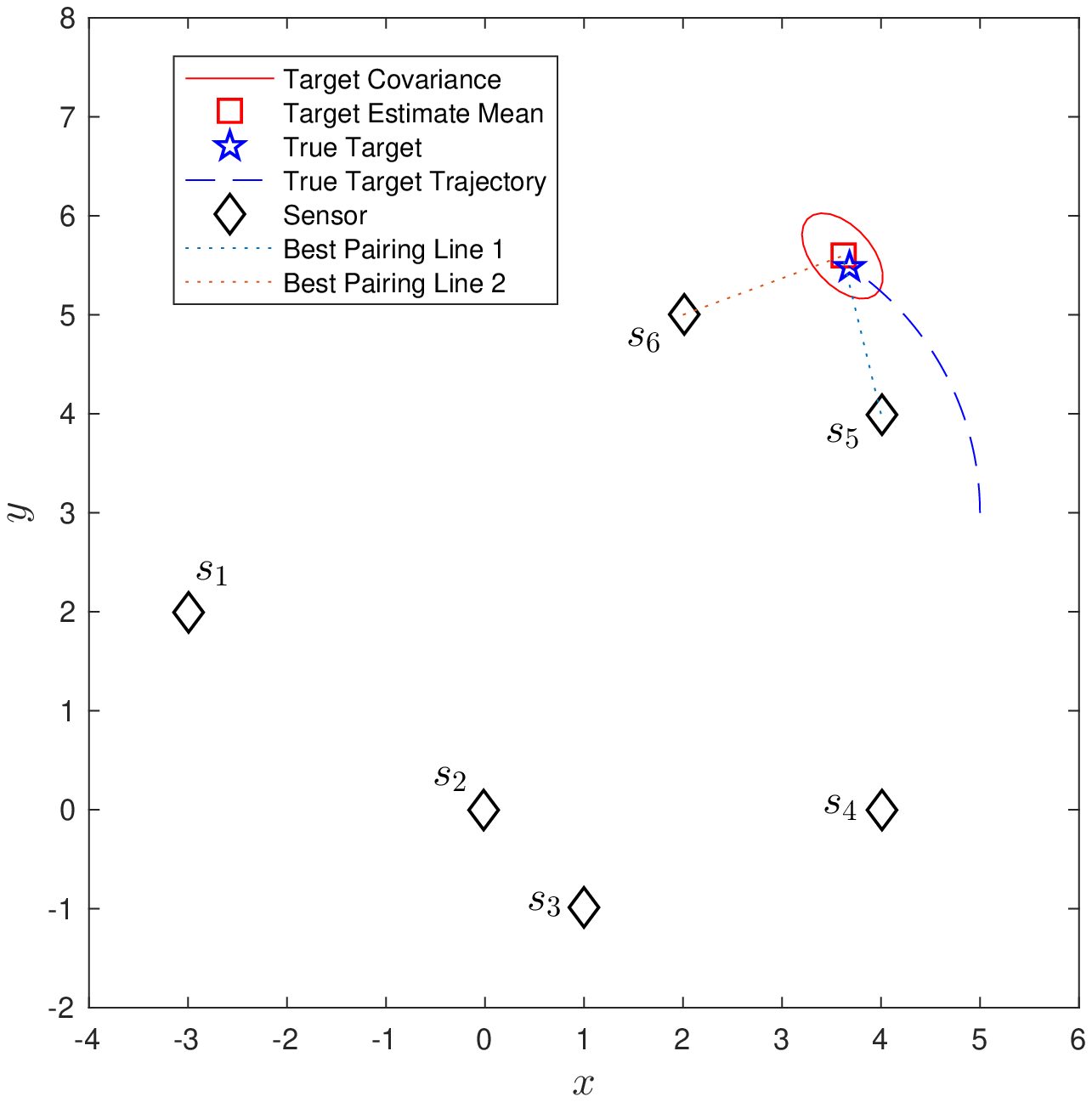}}
% \subfigure[$k=100$]{
% \includegraphics[width=0.49\columnwidth]{figs/two_flexible_k_100.eps}}
\subfigure[$k=150$]{
\includegraphics[width=0.66\columnwidth]{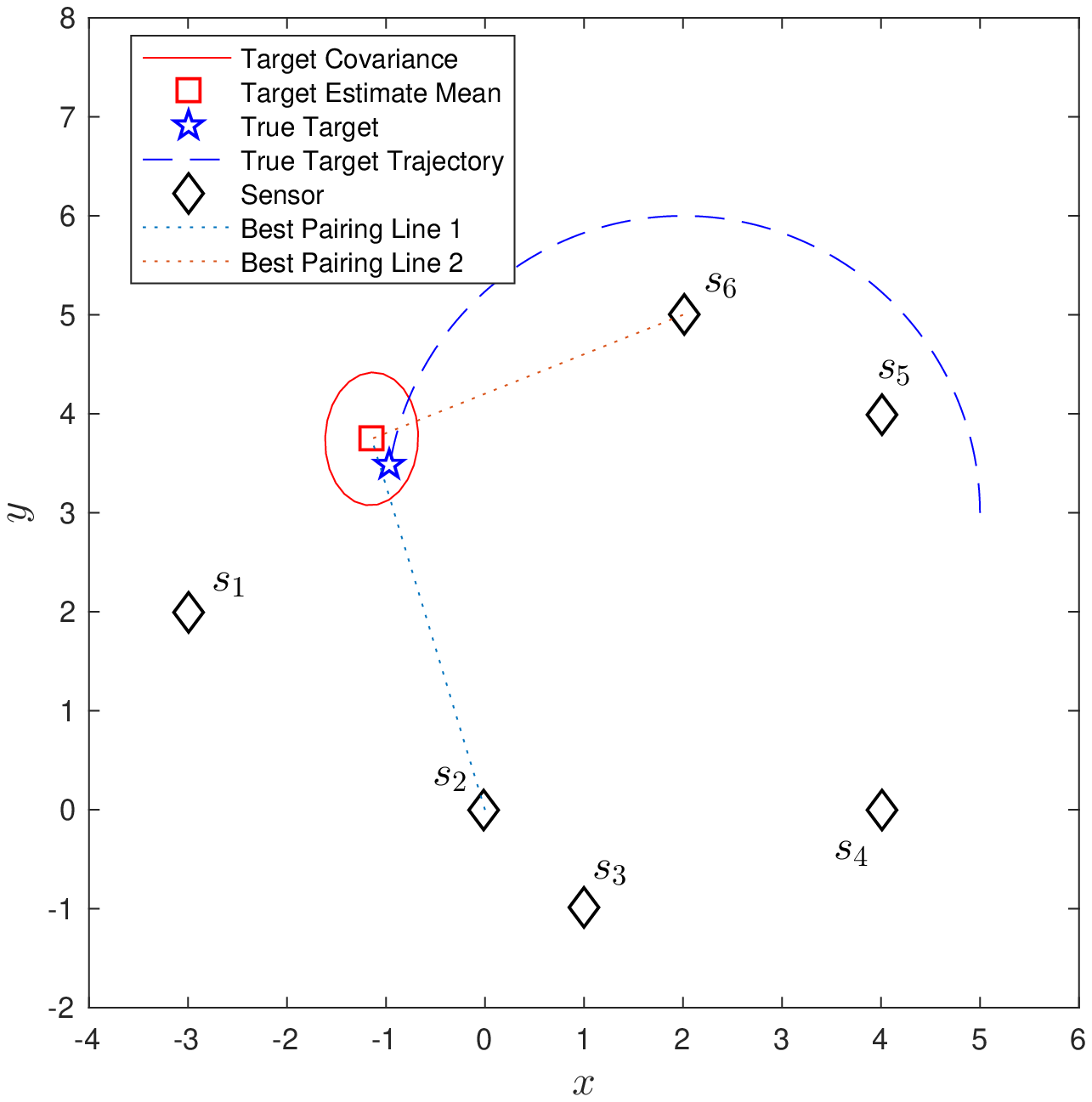}}
\subfigure[$k=200$]{
\includegraphics[width=0.66\columnwidth]{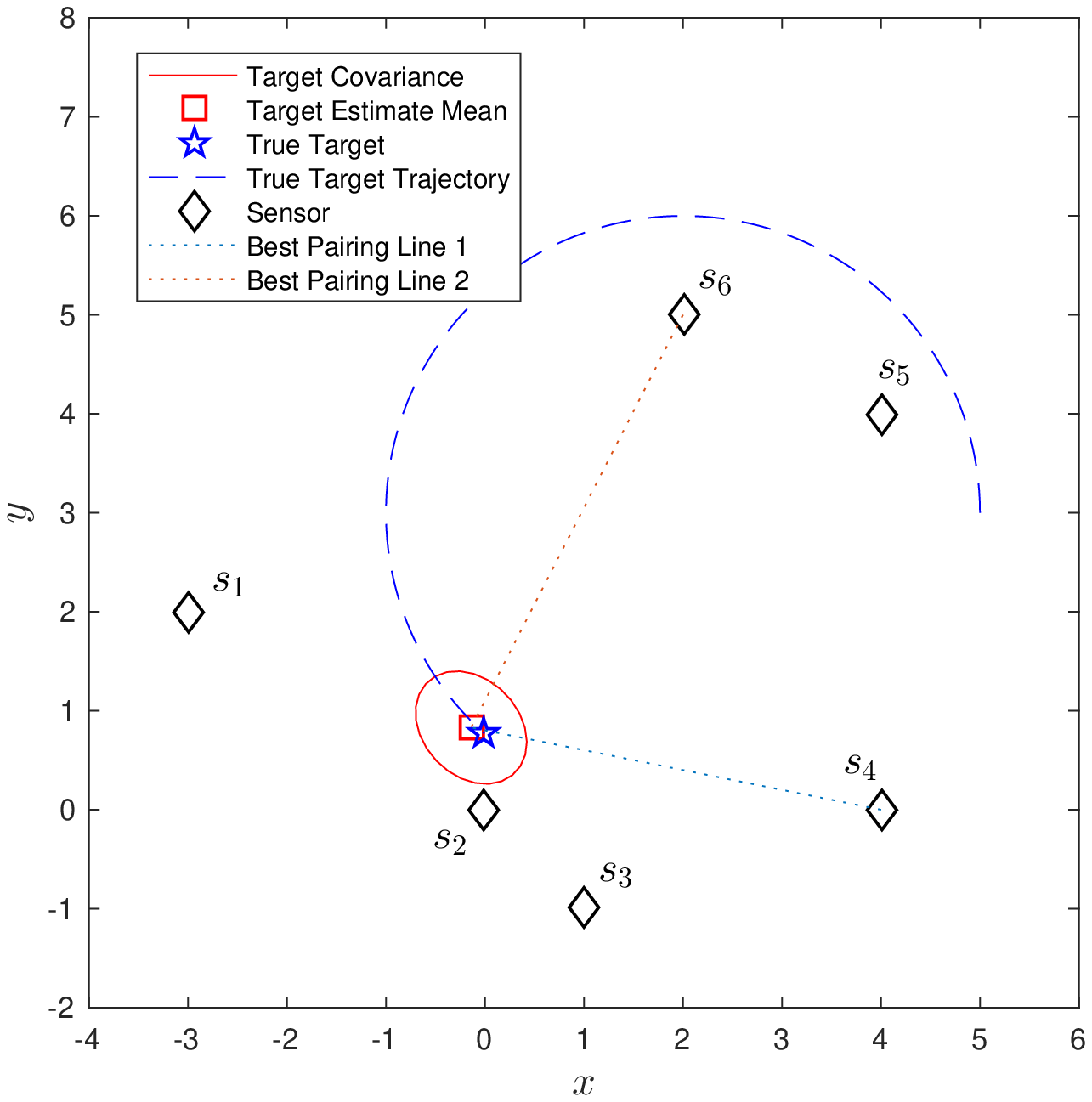}}
\subfigure[$k=250$]{
% \includegraphics[width=0.49\columnwidth]{figs/two_flexible_k_250.eps}}
% \subfigure[$k=300$]{
\includegraphics[width=0.66\columnwidth]{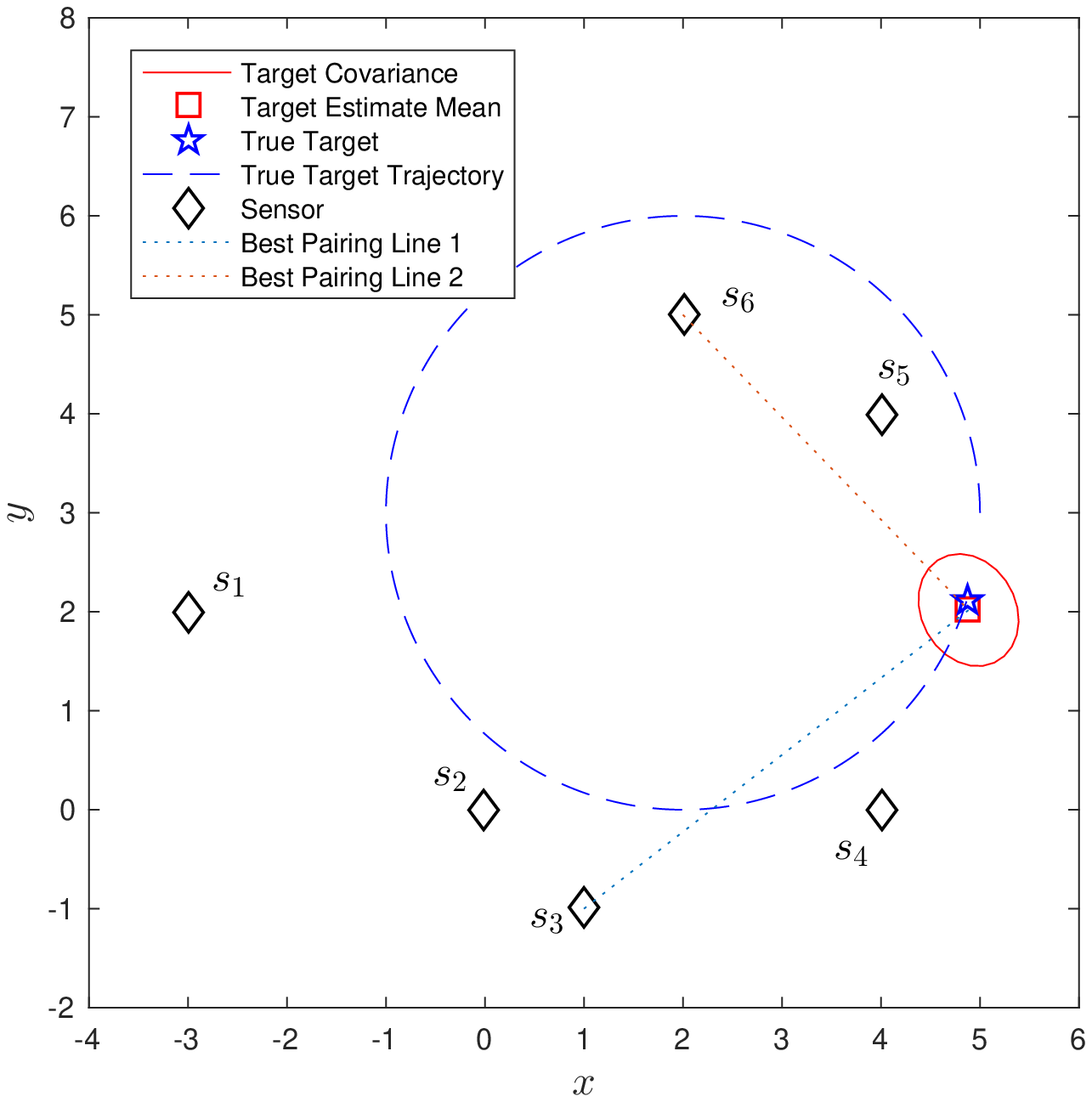}}
\subfigure[$k=315$]{
\includegraphics[width=0.66\columnwidth]{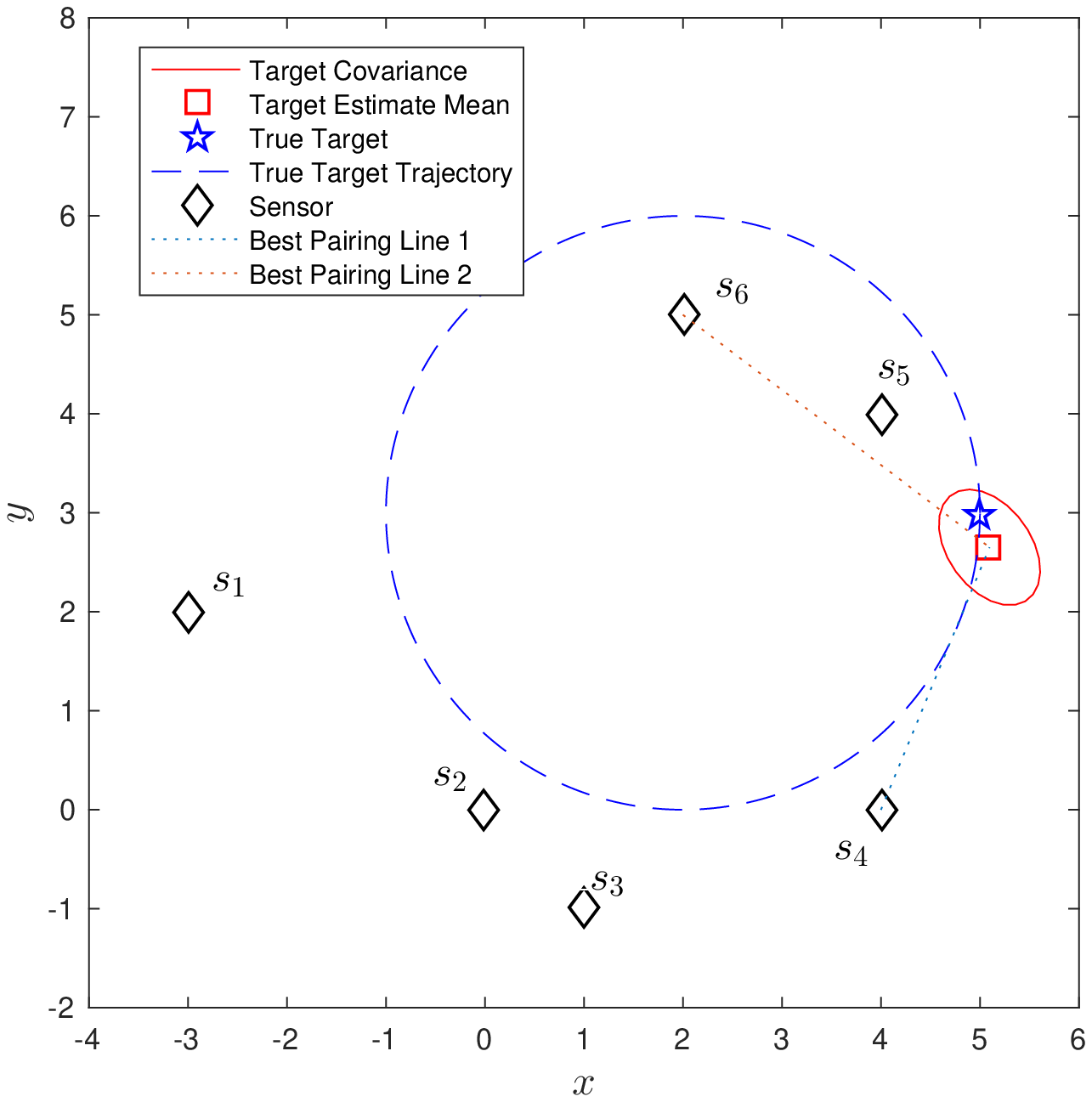}}
}
\caption{Strategy (1) in action for tracking the target with a circular motion.\label{fig:flexible_best_pairing}}
\end{figure*}

\begin{figure}[htb]
\centering{
\subfigure[]{
\includegraphics[width=0.47\columnwidth]{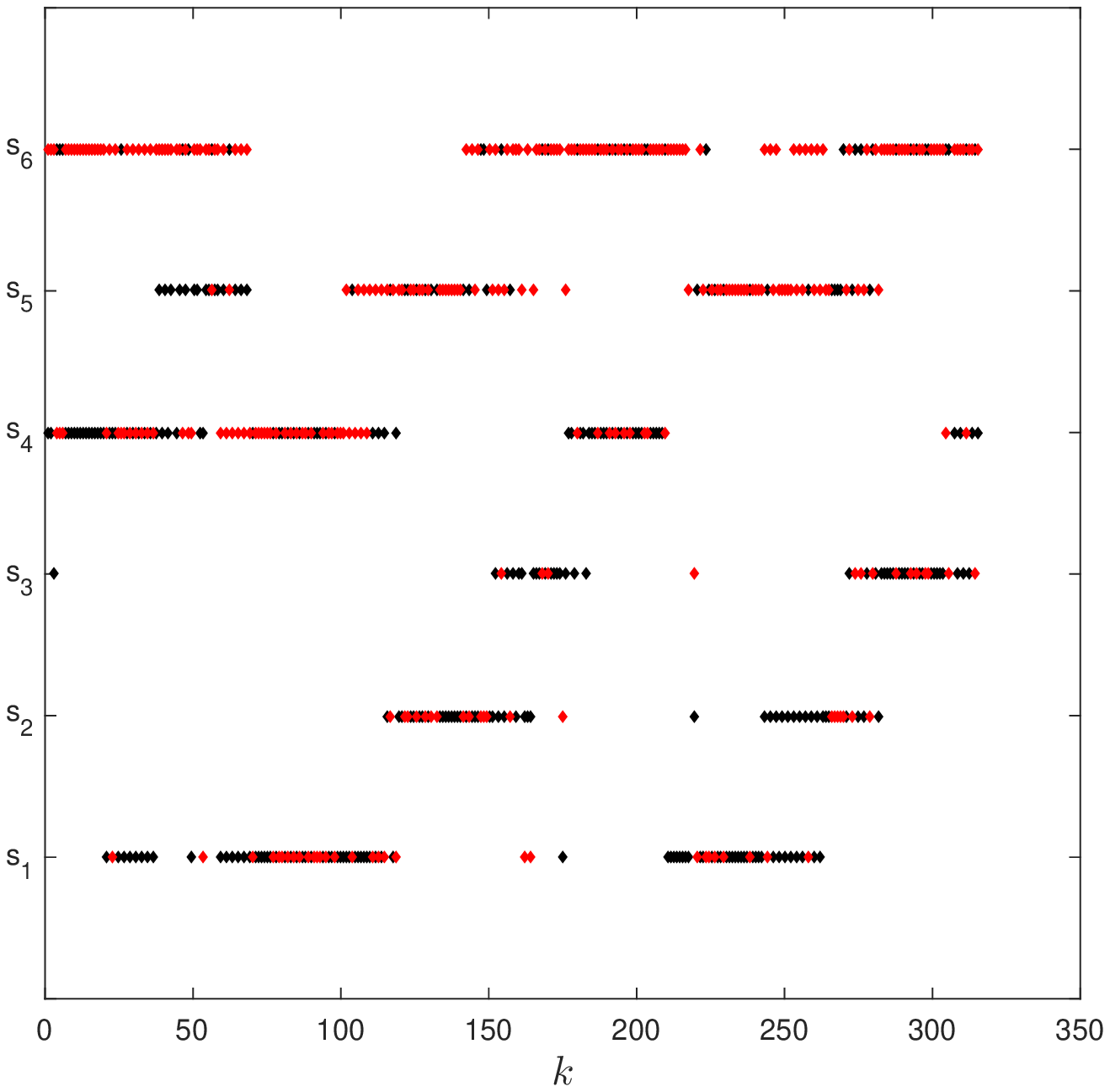}}
\subfigure[]{
\includegraphics[width=0.47\columnwidth]{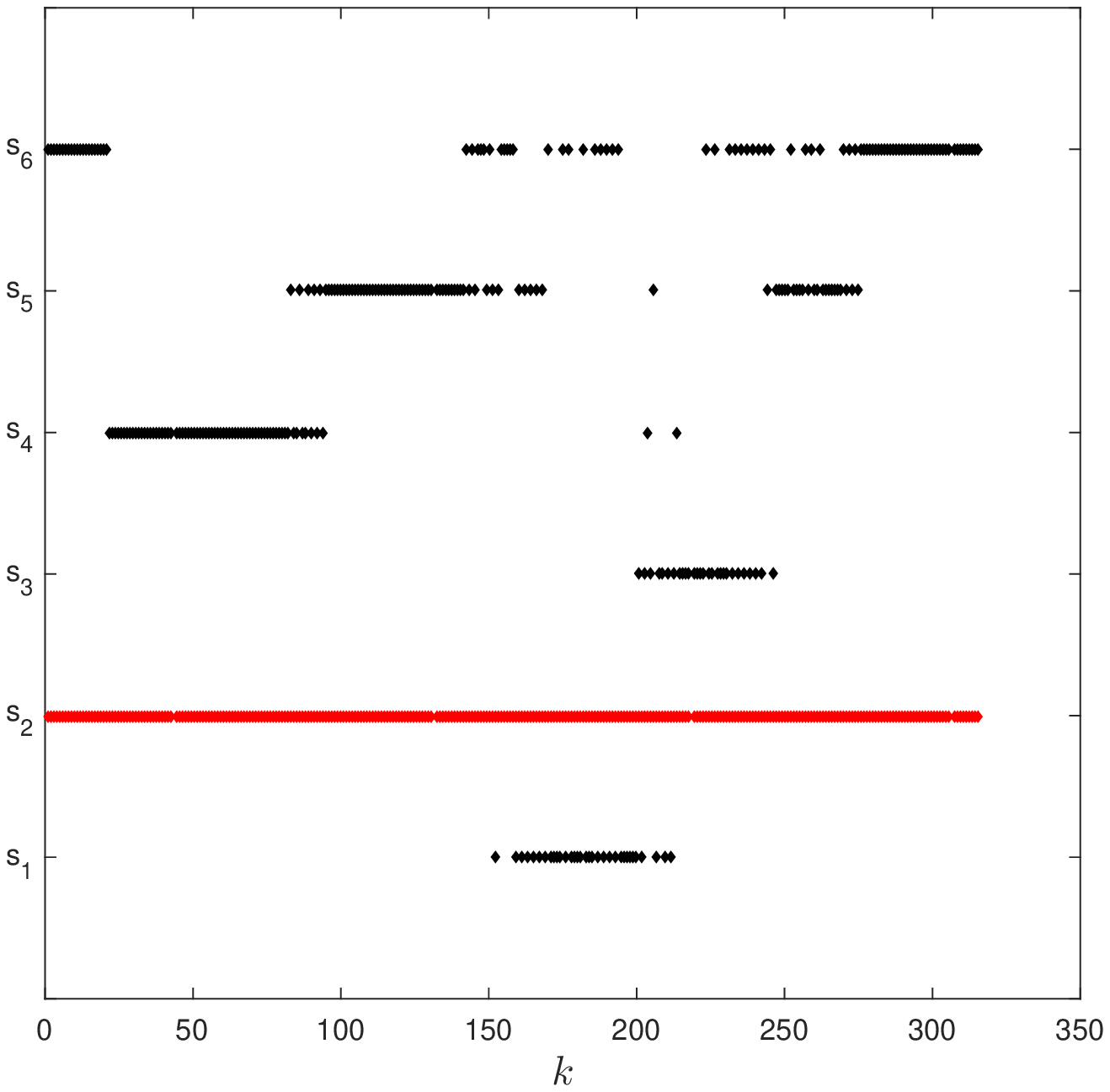}}
}
\caption{(a) Best pair selection at each timestep. (b) Best partner selection for sensor $s_2$ at each timestep.\label{fig:best_fixed_selection}}
\end{figure}

We first consider the scenario where the target is moving independently of six stationary sensors ($s_1$--$s_6$). The target follows a circular trajectory with $u_{o,\max}=3$. We compare three strategies: (1) flexible best pair: we pair sensors $(s_i,s_j)$ which have the maximum $\underline{C}^{-1}(O_{i,j}(\hat{o}(k-1),u_{o,\max}))$ for tracking the moving target $o$ (Figure~\ref{fig:best_fixed_selection}--(a));  (2) flexible pair for fixed sensor: we pair an arbitrarily chosen sensor, $s_2$, with the best sensor at each timestep (Figure \ref{fig:best_fixed_selection}--(b)); (3) best-fixed pair: we pair $s_2$ with $s_1$ at all timesteps. We chose $s_1$ since it gives the best performance amongst all other sensors, on an average, for the circular trajectory.

Figure~\ref{fig:flexible_best_pairing}--(a) to (f) shows the result of running strategy (1) over $315$ timesteps. The best selected pair at each timestep $k$ is shown in Figure \ref{fig:best_fixed_selection}-(a). 

We show the results for the three strategies in Figure~\ref{fig:compare_pair_threecase}. We plot the inverse of the condition number, the estimation error $\epsilon:=\|\hat{o}-o\|_{2}$, and the trace of covariance matrix over time. We observe that the first strategy maintains a higher $\underline{C}^{-1}(O_{i,j}(\hat{o},u_{o,\max}))$ and has lower $\epsilon$ and trace of covariance almost all times. Around $k=50$, the best pair changes frequently as seen in Figure \ref{fig:best_fixed_selection}-(a).

\begin{figure*}[htb]
\centering{
\subfigure[flexible best pairing]{\includegraphics[width=0.67\columnwidth]{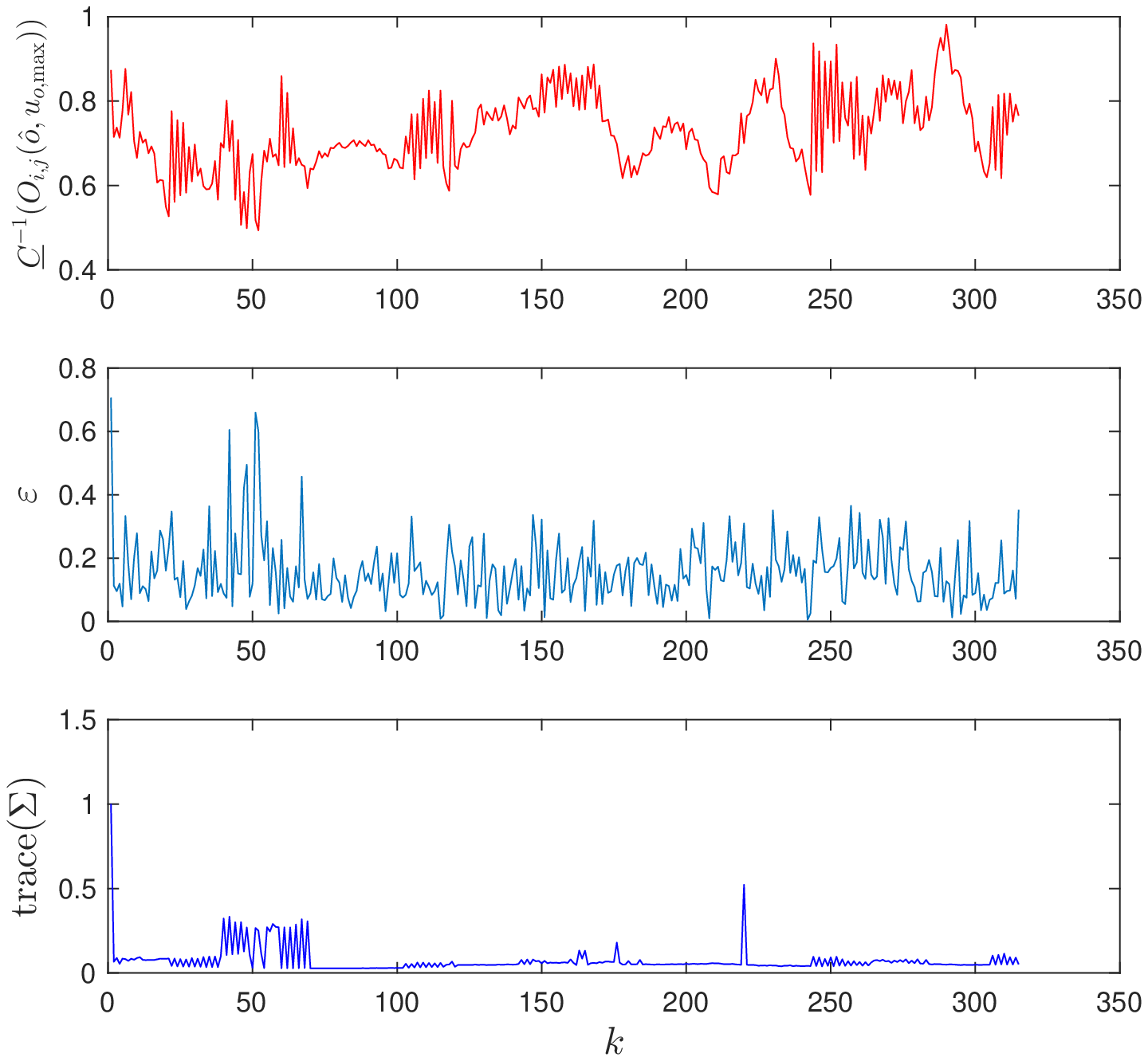}}
\subfigure[best pairing for a fixed sensor $s_2$]{\includegraphics[width=0.67\columnwidth]{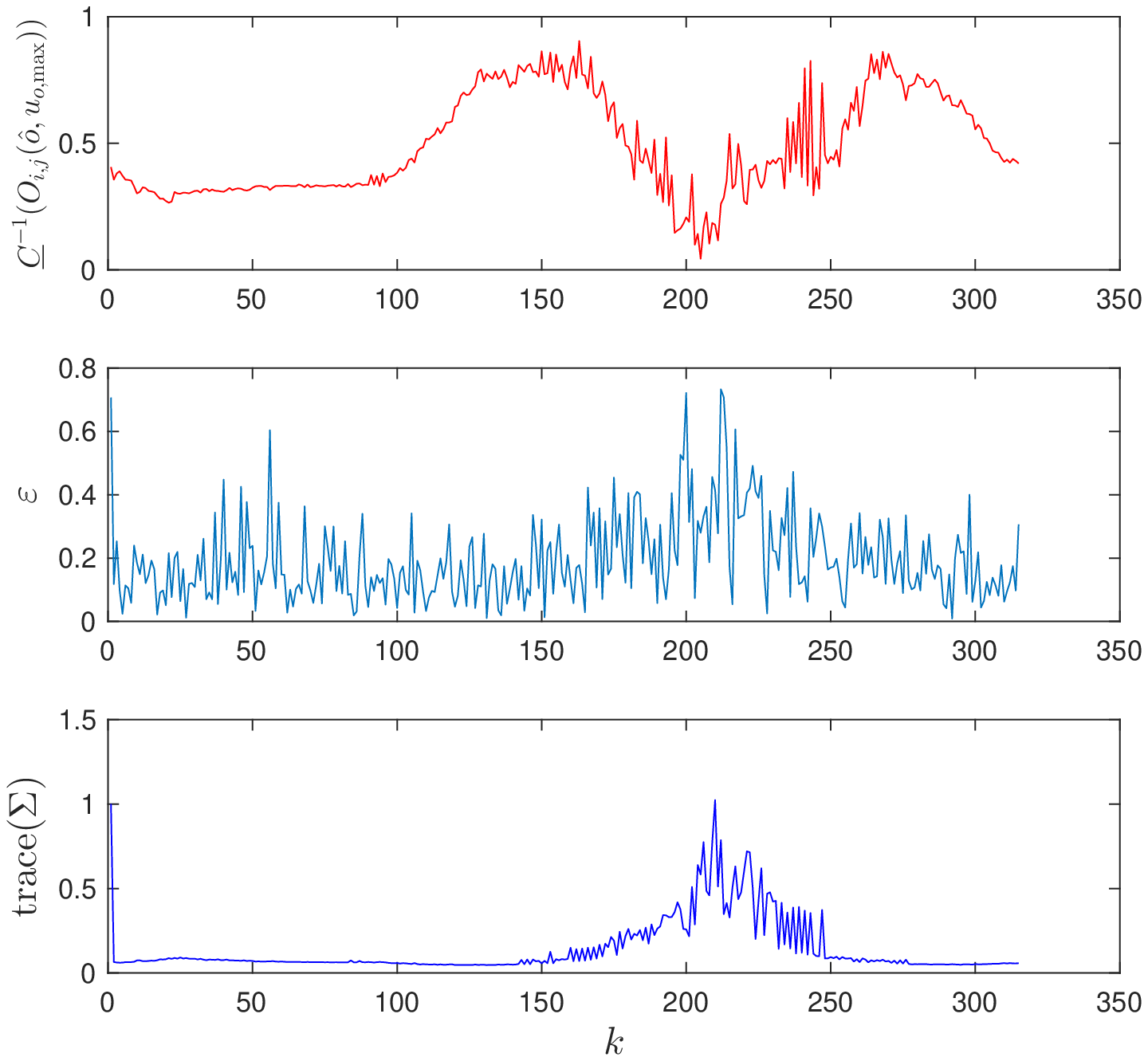}}
\subfigure[best fixed partner, $s_1$, for $s_2$]{\includegraphics[width=0.67\columnwidth]{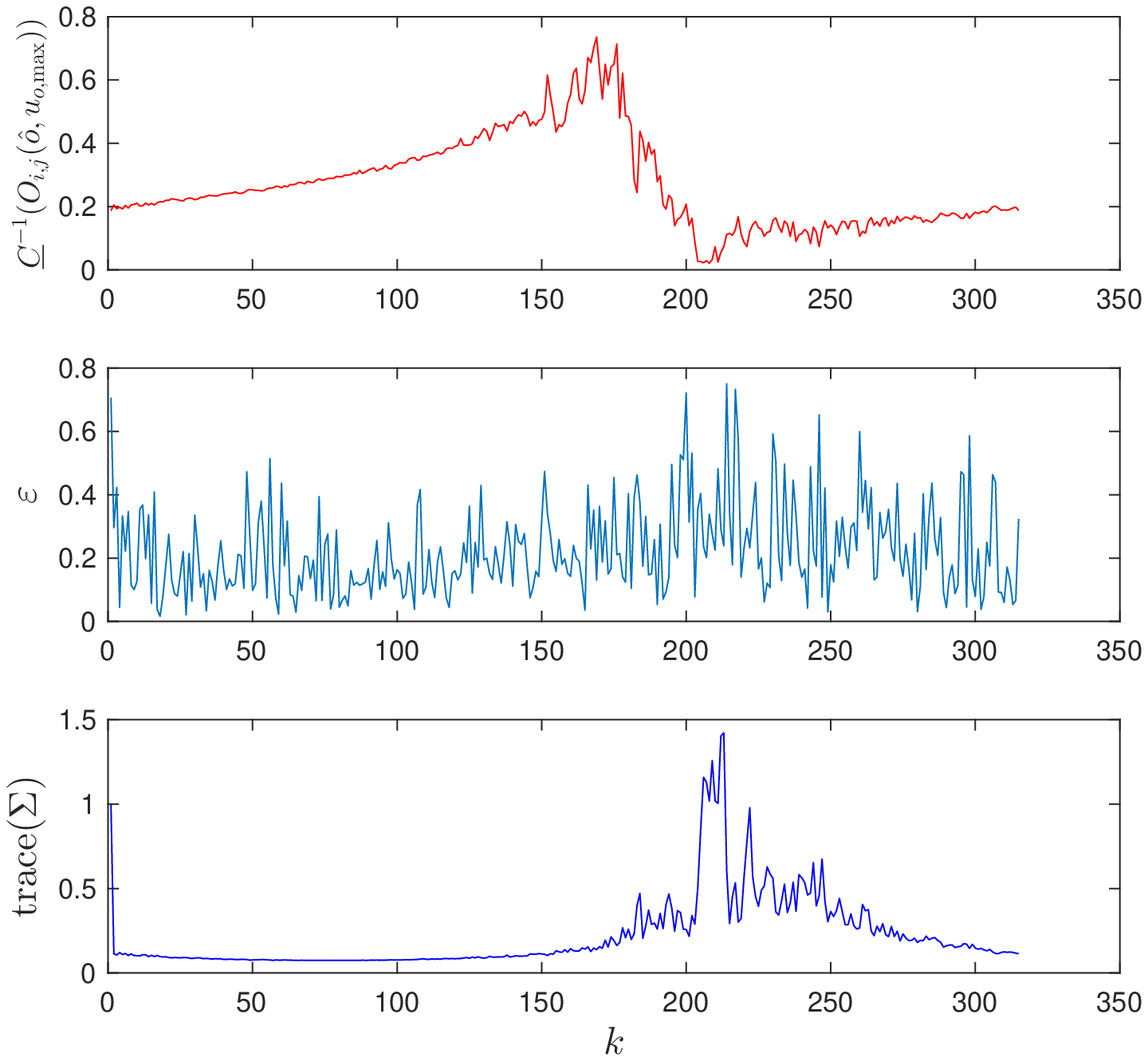}}
}
\caption{Comparison of lower bound of the inverse of condition number,  estimate error, and the trace of covariance matrix in three scenarios.\label{fig:compare_pair_threecase}} 
\end{figure*}

Note that in all three cases, higher $\underline{C}^{-1}(O_{i,j}(\hat{o},u_{o,\max}))$ correlates with lower $\epsilon$ and $\mathrm{trace}(\Sigma)$ suggesting that the tracking performance can be improved by improving the lower bound of the inverse of the condition number. 

\subsection{Adversarial Target} \label{sec:game}
We also simulate the scenario where the target moves in an adversarial fashion. At each time step, the target moves in a direction so as to minimize the inverse of the condition number. Here, the target has an advantage in that it knows its own exact state and control input, while sensors only have an estimate. At each timestep, the target evaluates all the possible control inputs within a ball of radius $u_{o,\max}$, around itself and chooses on that minimizes ${C}^{-1}(O_{i,j}(o,u_{o}))$.

% $u_o:=[u_{ox},u_{oy}]^{T}$ with,
% \begin{equation*}
% \left\{
%                 \begin{array}{ll}
%                   u_{ox}=u_{o,\max}\cos\beta,\\
%                   u_{oy}=u_{o,\max}\sin\beta, \beta\in[0,2\pi], 
%                 \end{array}
%               \right.
% \label{eqn:all_control_input}
% \end{equation*}
% to find one that minimizes ${C}^{-1}(O_{i,j}(o,u_{o}))$. In the simulation, the set of possible target's control input is computed by discretizing the control input disc with radius, $[0,u_{o,\max}]$ and angular range $[0, 2\pi]$. 

Figure \ref{fig:target_entrap} shows an instance where the target is ``trapped'' by the sensors. The target fails to reduce the inverse of the condition number since the sensor pair switches whenever the target gets closer to a sensor.

\begin{figure}[htb]
\centering{
\subfigure[]{
\includegraphics[width=0.9\columnwidth]{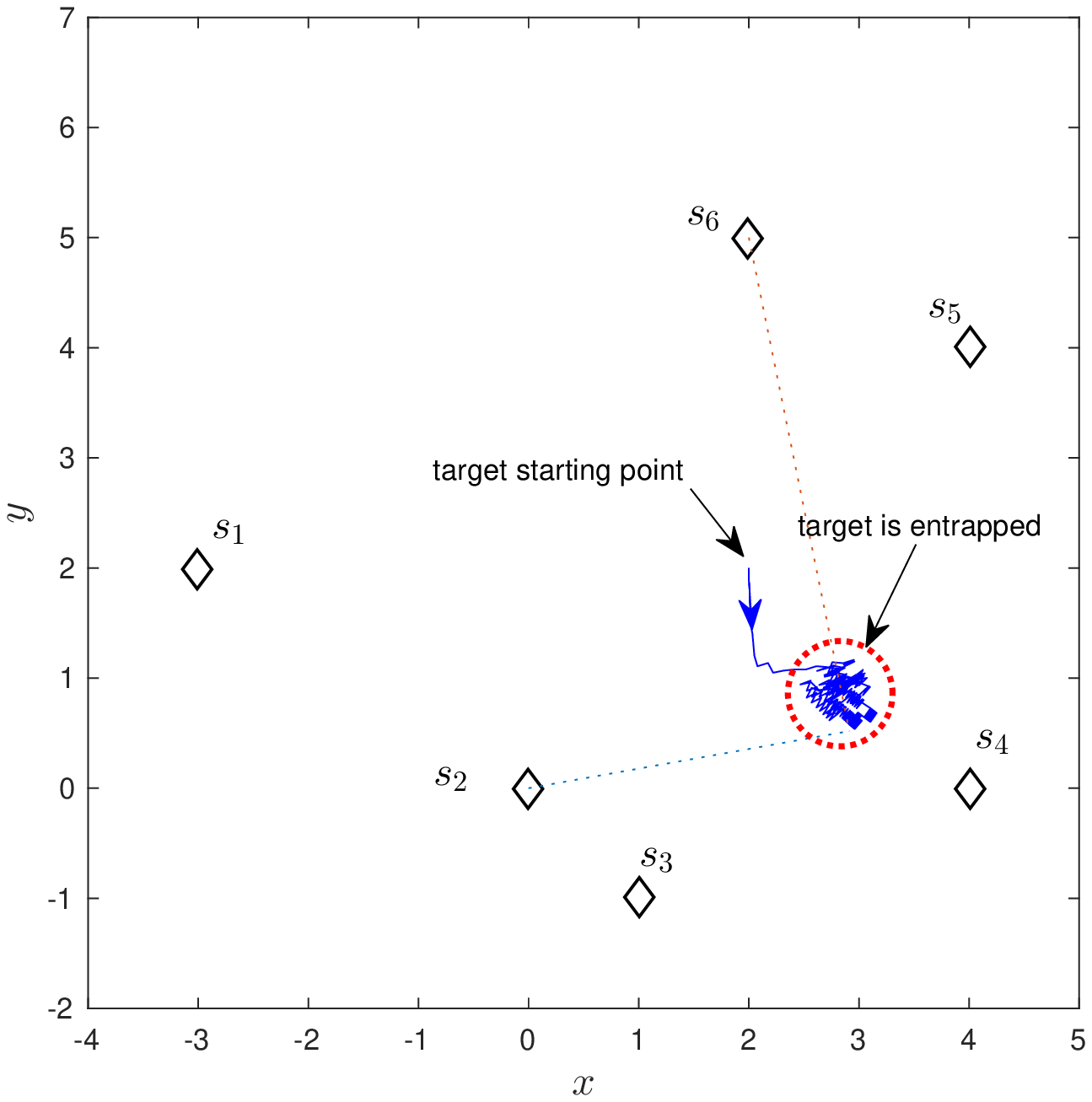}}
\subfigure[]{
\includegraphics[width=0.9\columnwidth]{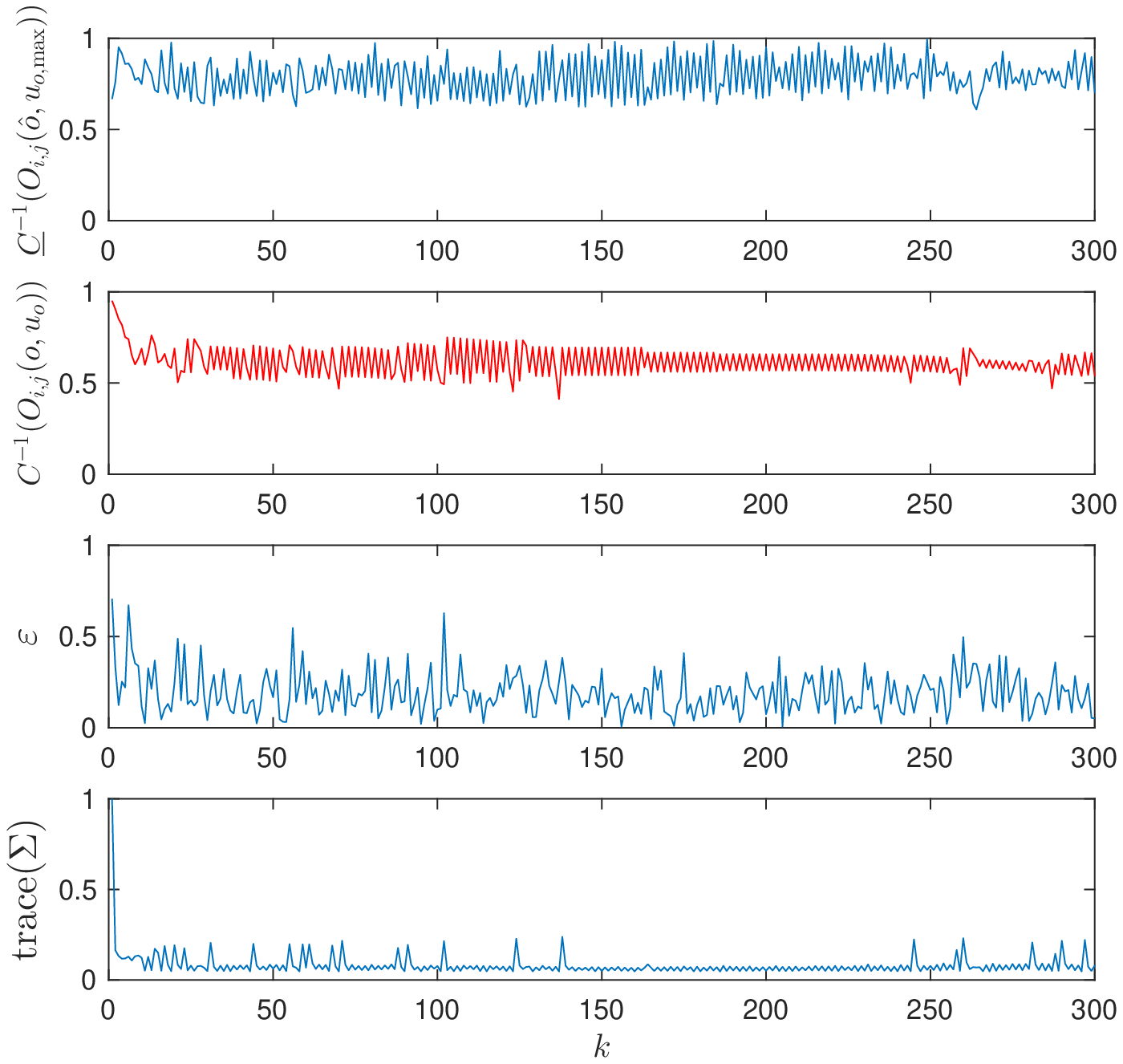}}
}
\caption{An adversarial target is ``entrapped'' by sensors using best pairing strategy. The inverse of the condition number remains at a high value. Even if the target moves closer to a sensor to reduce the inverse of the condition number, the sensor pair switches thereby increasing the inverse of the condition number.\label{fig:target_entrap}}
\end{figure}

\subsection{Greedy Unique Pair Assignment}
The Unique Pair Assignment problem is NP-Complete. Therefore, finding $\omega(\textrm{OPT})$ is infeasible in polynomial time. In order to empirically evaluate the Greedy Unique Pair Assignment (Algorithm~\ref{algorithm:unique_pair_assignment}), we resort to a new assignment scenario where each sensor can be matched in multiple different pairs and each sensor pair can be assigned to at most one specific target. We formulate the new assignment as Relaxed Pair Assignment (Problem \ref{prob:perfect}). It is clear that solving Relaxed Pair Assignment problem optimally gives us an upper bound of optimality for Unique Pair Assignment problem. We can use this upper bound for the comparison of the greedy approach in Unique Pair Assignment.
\begin{problem}[Relaxed Pair Assignment] Given a set of sensor positions, $\mathcal{S} := \{s_0,\ldots,s_N\}$ and a set of target estimates at time $t$, $\mathcal{T} := \{t_0,\ldots,t_L\}$, find an assignment of unique pairs of sensors to targets:
\begin{equation}
\text{maximize} \sum_{l=1}^L \omega(\sigma_1(l),\sigma_2(l), t_l)
\end{equation}
with the added constraint that all pairs are unique, that is, $\forall k, l= 1,\ldots, L$, $k\neq l$, $\sigma_1(k) \neq \sigma_1(l)$ and/or $\sigma_2(k) \neq \sigma_2(l)$. 
\label{prob:perfect}
\end{problem}

The Relaxed Pair Assignment problem can be solved optimally by using maximum weight perfect bipartite matching (MWPBM)~\cite{cormen2009introduction}. %and solved for the optimal solution (Algorithm \ref{algorithm:maximum_perfect_matching}). We firstly formulate this problem as a Maximum Weight Perfect Bipartite Matching Problem:
Note that a sensor can be matched in multiple distinct pairs and assigned to multiple targets. The MWPBM can be solved using the Hungarian algorithm~\cite{kuhn1955hungarian} in polynomial time. 
% Figure \ref{fig:maximum_weighted_matching} shows an instance with six sensors and five targets. For each target, the upper bound of  the control input is specified as $u_{o_{l},max}=1$. We use the implementation for the matching provided online~\cite{matching}. 
% \begin{figure}[htb]
% \centering{
% %\subfigure[]{\includegraphics[width=0.8\columnwidth]{figs/multi_robot_target.eps}}
% \includegraphics[width=0.8\columnwidth]{figs/multi_robot_multi_target_matching.eps}
% }
% \caption{Maximum weighted perfect matching with six sensors $s_1$--$s_6$ and five targets $t_1$--$t_5$ at one timestep.\label{fig:maximum_weighted_matching}} 
% \end{figure}

In order to see the effectiveness of the greedy algorithm in Unique Pair Assignment, we compare the total value charged by the greedy algorithm, $\omega(\textrm{GREEDY})$, with the total value charged by the MWPBM, $\omega(\textrm{MWPBM})$, in Relaxed Pair Assignment as shown in Figure~\ref{fig:comparsion_perfect_unique}. We consider different number of targets with $L$ for 1 to 20 and set the number of robots, $N=2L$. For each $L\in\{1,2,...,20\}$, the positions of sensors and targets are randomly generated within $[0,100]\times [0,100]\in\mathbb{R}^{2}$ for 30 trials. Set the maximum control input for each target as $u_{o,\max}=1$.  Figure~\ref{fig:comparsion_perfect_unique} shows that $\omega(\textrm{GREEDY})$ is close to $\omega(\textrm{MWPBM})$ and much higher than $\frac{1}{3}\omega(\textrm{MWPBM})$. Thus, even though we give a theoretical $1/3$--approximation for the greedy algorithm, it performs much better in practice.

\begin{figure}[htb]
\centering{
%\subfigure[]{\includegraphics[width=0.8\columnwidth]{figs/multi_robot_target.eps}}
\includegraphics[width=0.9\columnwidth]{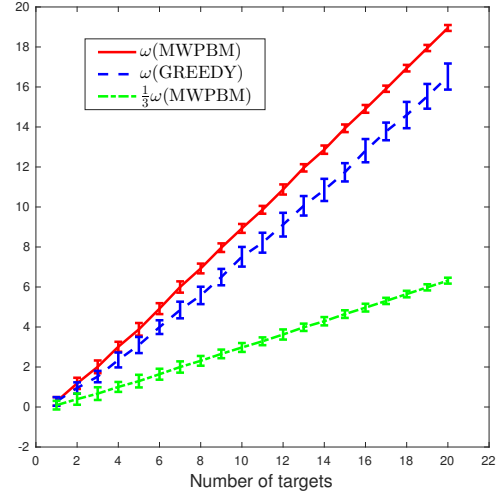}
}
\caption{Comparison of total reward collected by greedy approach (Algorithm~\ref{algorithm:unique_pair_assignment}) with the maximum perfect pair matching.\label{fig:comparsion_perfect_unique}} 
\end{figure}

\subsection{General Pair Assignment}
We also simulate the greedy assignment~\cite{nemhauser1978analysis} for the General Assignment problem by using log determinant of  \emph{symmetric observability matrix},  $\mathbb{O}(o,u_o)$ as observability measure. We set the number of targets as $L=5$ and number of sensors from 20 to 50. For each $N\in\{20,21,...,50\}$, the positions of sensors and targets are randomly generated within $[0,100]\times [0,100]\in\mathbb{R}^{2}$ for 30 trials. 
% We know finding the optimal solution of General Assignment problem is NP-Complete. 
We compare the number of sensors assigned to one specific target, i.e., $t_l, l\in\{1,2,...L\}$ with the $N/L$ as shown in Figure~\ref{fig:trial_submodular_obs.eps}. It shows that the sensors are assigned to each target almost evenly.  

\begin{figure}[htb]
\centering{
%\subfigure[]{\includegraphics[width=0.8\columnwidth]{figs/multi_robot_target.eps}}
\includegraphics[width=0.9\columnwidth]{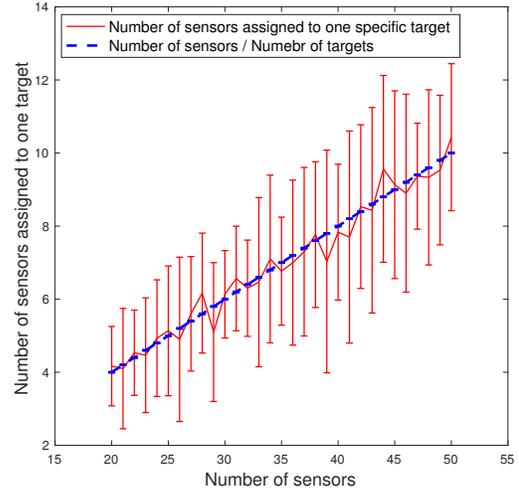}
}
\caption{Number of sensors assigned to each target.\label{fig:trial_submodular_obs.eps}} 
\end{figure}

\section{Conclusion}
\label{sec:Conclusion}
In this paper, we solved sensor assignment problems to improve the observability for target tracking. We derived the lower bound on the inverse of the condition number of the observability matrix for a system with a mobile target and $N$ stationary sensors. The lower bound considers only the known part of the observability matrix --- the sensor-target relative position and an upper bound on the target's speed. We showed how this lower bound can be employed for sensor selection. % In particular, we showed how improving the lower bound by selecting the optimal set of sensors improves the tracking performance. We also simulated an adversarial case where the target moves to minimize observability and the sensors choose pairs to maximize observability. The salient feature of this work is that the target is powerful since it knows the actual observability matrix whereas the sensors only have a lower bound. Nevertheless, we observe that the sensors are able to track the target well. 
We considered two sensor assignment problems for which we presented constant-factor approximation algorithm. Our immediate work is focused on assigning sensors to cover an area instead of tracking a group of targets. Another avenue is designing an efficient set covering strategy based on observability measures which are submodular and monotone. 
%A limitation of our assignment algorithm is that a single sensor may be tasked to track \pcomment{many:weasel} targets thereby expending significant communication energy. Our immediate work is focused on treating the sensor assignment problem as a cooperative game with selfish interest and a social goal.

\bibliographystyle{IEEEtran}
\bibliography{IEEEabrv,refs}
%%%% appendix
\appendix
\section*{Proof for Theorem~\ref{thm:lower_bound}}
\begin{proof}
We partition $O(o,u_o)$ into the known and unknown parts as, 
\begin{eqnarray}
O(o,u_o)=\begin{bmatrix}
    O(o) \\ 
     O(u_o)  
\end{bmatrix},
\label{eqn:new_multi_sensor_target_partition}
\end{eqnarray}
where
\begin{equation}
O(o):=\begin{bmatrix}
    o_{x}-p_{1x},o_{y}-p_{1y} \\ 
     o_{x}-p_{2x},o_{y}-p_{2y} \\ 
    \vdots\\
    o_{x}-p_{Nx},o_{y}-p_{Ny}
\end{bmatrix},
\label{eqn: state_contribution_observability}
\end{equation}
and 
\begin{equation}
O(u_o):=\begin{bmatrix}
     u_{ox},u_{oy}   
\end{bmatrix},
\label{eqn: control_contribution_observability}
\end{equation}
indicate the contribution to the observability matrix from sensor-target relative state and target's control input, respectively. 
%Here, since $u_o$ is fixed, unknown and not controllable by the sensor, and thus $O(u_o)$ is not controllable by the sensor, either.  We focus on the relative state contribution part $O(o)$. Observability of trajectories (i.e., with control inputs) is not trivial (see for example \cite{gadre2004toward} for one sensor and one target case) and is the focus of our ongoing research.

The singular values of $O(o,u_o)$ can be found as the square-root of the eigenvalues of the \emph{symmetric observability matrix},  $\mathbb{O}(o,u_o)$, given as~\cite{Strang1976Linear},
\begin{eqnarray}
\mathbb{O}(o,u_o) &=&O^{T}(o,u_o)O(o,u_o)\nonumber\\
&=&O^{T}(o)O(o)+O^{T}(u_o)O(u_o).
\label{sy_ob_matrix}
\end{eqnarray}
\begin{equation}
\left\{
                \begin{array}{ll}
                  \sqrt{{\lambda_{\min}(\mathbb{O}(o,u_o))}}={\sigma_{\min}(O(o,u_o))},\\
                 \sqrt{{\lambda_{\max}(\mathbb{O}(o,u_o))}}={\sigma_{\max}(O(o,u_o))}.
                \end{array}
              \right.
\label{eqn:min/max_singualr_observability_metrix}
\end{equation}

We can use \emph{Weyl and dual Weyl} inequalities to bound the singular values. For Hermitian matrices $X$ and $Y$ with $r$ eigenvalues written in increasing order $\lambda_1(X)\leq \lambda_2(X)\leq\ldots\leq\lambda_r(X)$ and $\lambda_1(Y)\leq \lambda_2(Y)\leq\ldots\leq\lambda_r(Y)$, respectively, the \emph{Weyl inequalities}\cite{franklin2012matrix} is given by,
\begin{equation}
\lambda_{i+j-1}(X+Y)\geq \lambda_{i}(X)+\lambda_{j}(Y)
\label{eqn:Weyl_inequalities}
\end{equation}
where $i,j\geq 1$ and $i+j-1\leq r$. Similarly, the \emph{dual Weyl inequalities} is given by
\begin{equation}
\lambda_{i+j-r}(X+Y)\leq \lambda_{i}(X)+\lambda_{j}(Y)
\label{eqn:Weyl_inequalities}
\end{equation}
where $i,j\geq 1$ and $i+j-r\leq r$.

Since $\mathbb{O}(o,u_o)\in \mathbb{R}^{2\times 2}$, $O^{T}(o)O(o)\in \mathbb{R}^{2\times 2}$ and $O^{T}(u_o)O(u_o)\in \mathbb{R}^{2\times 2}$ are symmetric matrices, they are Hermitian with the eigenvalues (in ascending order) as $\lambda_1(\mathbb{O}(o,u_o))\leq \lambda_2(\mathbb{O}(o,u_o))$, $\lambda_1(O^{T}(o)O(o))\leq \lambda_2(O^{T}(o)O(o))$ and $\lambda_1(O^{T}(u_o)O(u_o))\leq \lambda_2(O^{T}(u_o)O(u_o))$. Following the \emph{Weyl and dual Weyl} inequalities, we get
\begin{align}
&\lambda_1(\mathbb{O}(o,u_o)) \geq \lambda_1(O^{T}(o)O(o)) + \lambda_1(O^{T}(u_o)O(u_o)),\nonumber\\
&\lambda_1(\mathbb{O}(o,u_o)) \leq \nonumber\\ 
&\min\left\{
                \begin{array}{ll}
                 \lambda_1(O^{T}(o)O(o)) + \lambda_2(O^{T}(u_o)O(u_o))\\
                  \lambda_2(O^{T}(o)O(o)) + \lambda_1(O^{T}(u_o)O(u_o))
                \end{array},
              \right. \nonumber\\
&\lambda_2(\mathbb{O}(o,u_o)) \leq
\lambda_2(O^{T}(o)O(o)) + \lambda_2(O^{T}(u_o)O(u_o)),\nonumber\\
&\lambda_2(\mathbb{O}(o,u_o)) \geq \nonumber\\
& \max\left\{
                \begin{array}{ll}
                 \lambda_1(O^{T}(o)O(o)) + \lambda_2(O^{T}(u_o)O(u_o))\\
                  \lambda_2(O^{T}(o)O(o)) + \lambda_1(O^{T}(u_o)O(u_o))
                \end{array}.
              \right.
\label{eqn:Weyl_eigen_four}
\end{align}
Thus, 
\begin{align}
&\frac{\lambda_{1}(\mathbb{O}(o,u_o))}{\lambda_{2}(\mathbb{O}(o,u_o))}\geq
\nonumber\\
&\frac{\lambda_1(O^{T}(o)O(o)) + \lambda_1(O^{T}(u_o)O(u_o))}{\lambda_2(O^{T}(o)O(o)) + \lambda_2(O^{T}(u_o)O(u_o))}.
\label{eqn:lower_bound_inequality}
\end{align}
Then, from Equation \ref{eqn:min/max_singualr_observability_metrix} and Equation \ref{eqn:lower_bound_inequality},  the inverse of the condition number of the local nonlinear observability matrix,
\begin{align*}
\label{eqn:lower_bound_inversecond}
&C^{-1}(O(o,u_o))=\sqrt{\frac{\lambda_{1}(\mathbb{O}(o,u_o))}{\lambda_{2}(\mathbb{O}(o,u_o))}}\nonumber\\
&\geq \sqrt{\frac{\lambda_1(O^{T}(o)O(o)) + \lambda_1(O^{T}(u_o)O(u_o))}{\lambda_2(O^{T}(o)O(o)) + \lambda_2(O^{T}(u_o)O(u_o))}}.\nonumber\\
\end{align*}
By calculating the eigenvalues of symmetric matrix of  target's control contribution,
$$O^{T}(u_o)O(u_o)=\begin{bmatrix}
     u_{ox}^{2} & u_{ox}u_{oy}\\
    u_{oy}u_{ox} & u_{oy}^{2}
\end{bmatrix},$$
we get, 
\begin{eqnarray}
&&\lambda_{1}(O^{T}(u_o)O(u_o))=0\nonumber\\
&&\lambda_{2}(O^{T}(u_o)O(u_o))=u_{ox}^{2}+u_{oy}^{2}=u_{o}^{2}.
\label{eqn:max_min_eigenu}
\end{eqnarray}
\noindent Then the lower bound of $C^{-1}(O(o,u_o))$ is calculated as  \begin{eqnarray}
%\label{eqn:lower_bound_inversecond}
\underline{C}^{-1}(O(o,u_o)) &=&  \sqrt{\frac{\lambda_1(O^{T}(o)O(o))}{\lambda_2(O^{T}(o)O(o)) + u_{o}^{2}}}\nonumber\\
&=& \frac{\sigma_{\min}(O(o))}{\sqrt{\sigma^{2}_{\max}(O(o))+ u_{o}^{2}}}
\label{eqn:inverse_lower_bound_simply}
\end{eqnarray}
Equation~\ref{eqn:inverse_lower_bound_simply} gives the main lower bound. Note that $C^{-1}(O(o,u_o))$ cannot be determined since target's control input, $u_o$, is unknown. However, we know that $||u_o||_2 \leq u_{o,\max}$. Therefore,
\begin{equation}
%\label{eqn:lower_bound_inversecond}
\underline{C}^{-1}(O(o,u_o)) \geq \frac{\sigma_{\min}(O(o))}{\sqrt{\sigma^{2}_{\max}(O(o))+ u_{o,\max}^{2}}}
\label{eqn:inverse_lower_bound_withumax}
\end{equation}
This yields our main lower bound result.
\label{proof:theorem1}
\end{proof}
\section*{Proof for Theorem~\ref{thm:sensor_n_1}}
\begin{proof}
The local observability matrix for one-sensor-target, $i-o$ system can be derived from Equation \ref{eqn:new_multi_sensor_target} as,
\begin{equation}
O_{i}(o,u_o)=\begin{bmatrix}
    o_{x}-p_{ix} & o_{y}-p_{iy} \\ 
     u_{ox} & u_{oy}   
\end{bmatrix}.
\label{eqn:observability_i_o}
\end{equation}
The sensor-target relative state contribution is.
$$O_{i}(o)=\begin{bmatrix}
    o_{x}-p_{ix} & o_{y}-p_{iy}
\end{bmatrix}.$$
% It is easy to see that the single sensor cannot improve the lower bound of the observability matrix.
The $i-o$ system is weakly locally observable if $O_{i}(o,u_o)$ has full column rank, i.e., $(o_{x}-p_{ix})u_{oy} \neq (o_{y}-p_{iy})u_{ox}$. However, the sensor does not know the target's control input, $u_o$. %Observability of trajectories (i.e., with control inputs) is not trivial (see for example \cite{gadre2004toward} for one sensor and one target case) and is the focus of our ongoing research.

% both the rank and the degree of the observability is uncontrollable by the sensors based on the proofs of both \textbf{Theorem 1} and \textbf{Theorem 2}.
% Consider the number of sensors, $N=1$. The relative state observability matrix with one sensor $i$ and target $o$ can be described as 

% $$O_{i}(o)=\begin{bmatrix}
%     o_{x}-p_{ix} & o_{y}-p_{iy}
% \end{bmatrix},$$
% \noindent and
From the eigenvalues of symmetric matrix of sensor-target relative state contribution of $i-o$ system given by,
$$O_{i}^{T}(o)O_{i}(o)=\begin{bmatrix}
    (o_{x}-p_{ix})^{2},(o_{x}-p_{ix})(o_{y}-p_{iy}) \\ 
     (o_{x}-p_{ix})(o_{y}-p_{iy}), (o_{y}-p_{iy})^{2} 
\end{bmatrix},$$
we get, 
\begin{equation}
\left\{
                \begin{array}{ll}
                  {\sigma_{\min}(O_{i}(o))} =\sqrt{{\lambda_{\min}(O_{i}^{T}(o)O_{i}(o))}}=0,\\
                 {\sigma_{\max}(O_{i}(o))} =\sqrt{{\lambda_{\max}(O_{i}^{T}(o)O_{i}(o))}}\\
~~~~~~~~~~~~~~~=\sqrt{(o_{x}-p_{ix})^{2}+(o_{y}-p_{iy})^{2}}.
                \end{array}
              \right.
\label{eqn:min/max_singualr_observability_metrix_relative_state}
\end{equation}
Thus, from Equation \ref{eqn:inverse_lower_bound_simply}, the lower bound for $C^{-1}(O_{i}(o,u_o))$ is $\frac{\sigma_{\min}(O_{i}(o))}{\sqrt{\sigma^{2}_{\max}(O_{i}(o))+ u_{o}^{2}}}=0$. Consequently, the lower bound cannot be controlled by the sensor. 
\end{proof}
\section*{Proof for Theorem~\ref{thm:sensor_n_gre2}}
\begin{proof} Recall that ${\sigma_{\min}(O(o,u_o))}=\sqrt{{\lambda_{\min}(\mathbb{O}(o,u_o))}}$. We have,
\begin{equation}
\frac{\sigma_{\min}(O(o,u_o))}{\sigma_{\max}(O(o,u_o))}=\frac{\sqrt{\lambda_{\min}(\mathbb{O}(o,u_o))}}{\sqrt{\lambda_{\max}(\mathbb{O}(o,u_o))}}.
\label{eqn:eigensquare_observability_metrix}
\end{equation}
Following Equations~\ref{eqn:lower_bound_inequality} and \ref{eqn:max_min_eigenu}, we have the lower bound for $\frac{\lambda_{\min}(\mathbb{O}(o,u_o))}{\lambda_{\max}(\mathbb{O}(o,u_o))}$, described as
\begin{eqnarray}
\frac{\lambda_{\min}(\mathbb{O}(o,u_o))}{\lambda_{\max}(\mathbb{O}(o,u_o))}\geq \frac{\lambda_{\min}(O^{T}(o)O(o))}{\lambda_{\max}(O^{T}(o)O(o)) + u_{o}^{2}},
% \frac{\lambda_{1}(\mathbb{O}(o,u_o))}{\lambda_{2}(\mathbb{O}(o,u_o))}\leq \frac{\min\left\{
%                 \begin{array}{ll}
%                  \lambda_1(O^{T}(o)O(o)) + \lambda_2(O^{T}(u_o)O(u_o))\\
%                   \lambda_2(O^{T}(o)O(o)) + \lambda_1(O^{T}(u_o)O(u_o))
%                 \end{array}
%               \right. }{\max\left\{
%                 \begin{array}{ll}
%                  \lambda_1(O^{T}(o)O(o)) + \lambda_2(O^{T}(u_o)O(u_o))\\
%                   \lambda_2(O^{T}(o)O(o)) + \lambda_1(O^{T}(u_o)O(u_o))
%                 \end{array}
%               \right.}.
\label{eqn:upper_lower_bounds}
\end{eqnarray}
The observability can be improved by increasing this lower bound. We can transform the statement of the theorem (using eigenvalues instead of singular values) as:
\emph{if} 
\begin{equation}
\left\{
                \begin{array}{ll}
                 \frac{\lambda_{\min}^{'}(O^{T}(o)O(o))}{\lambda_{\max}^{'}(O^{T}(o)O(o))}\geq\frac{\lambda_{\min}(O^{T}(o)O(o))}{\lambda_{\max}(O^{T}(o)O(o))},\\
                \lambda_{\min}^{'}(O^{T}(o)O(o))\geq \lambda_{\max}(O^{T}(o)O(o))
                \end{array}
              \right.
\label{eqn:assumption_singular_condition}
\end{equation}
\emph{then} 
\begin{eqnarray}
&&\frac{\lambda_{\min}^{'}(O^{T}(o)O(o))}{\lambda_{\max}^{'}(O^{T}(o)O(o)) + u_{o}^{2}} \nonumber\\
&&\geq \frac{\lambda_{\min}(O^{T}(o)O(o))}{\lambda_{\max}(O^{T}(o)O(o)) + u_{o}^{2}}
\label{eqn:conslusion_singular_condition}
\end{eqnarray}
where the $\lambda$ and $\lambda'$ denotes the eigenvalues before and after the sensors' apply their control, respectively.

We start with the left-hand side of Equation~\ref{eqn:conslusion_singular_condition}) to get,
\begin{eqnarray}
&&\frac{\lambda_{\min,o}^{'}}{\lambda_{\max,o}^{'} + u_{o}^{2}}-\frac{\lambda_{\min,o}}{\lambda_{\max,o} + u_{o}^{2}}\nonumber\\
&=&\frac{\lambda_{\min,o}^{'}\lambda_{\max,o}-\lambda_{\max,o}^{'}\lambda_{\min,o}}{(\lambda_{\max,o}^{'} + u_{o}^{2})(\lambda_{\max,o} + u_{o}^{2})}\nonumber\\
&+&\frac{u_{o}^{2}(\lambda_{\min,o}^{'}-\lambda_{\min,o})}{(\lambda_{\max,o}^{'} + u_{o}^{2})(\lambda_{\max,o} + u_{o}^{2})}
\label{eqn:proof}
\end{eqnarray}
where $\lambda_{\min,o}^{'}$, $\lambda_{\max,o}^{'}$, $\lambda_{\min,o}$ and $\lambda_{\max,o}$ indicate $\lambda_{\min}^{'}(O^{T}(o)O(o))$, $\lambda_{\max}^{'}(O^{T}(o)O(o))$, $\lambda_{\min}(O^{T}(o)O(o))$ and $\lambda_{\max}(O^{T}(o)O(o))$. 

Using Equation~\ref{eqn:assumption_singular_condition}, it is easy to show 
\begin{equation}
\frac{\lambda_{\min,o}^{'}}{\lambda_{\max,o}^{'} + u_{o}^{2}}-\frac{\lambda_{\min,o}}{\lambda_{\max,o} + u_{o}^{2}} \geq 0.
\label{eqn:conclusion_proof}
\end{equation}
Hence the claim is proved.
\end{proof}
\begin{rem}
Equation~\ref{eqn:assumption_singular_condition} is sufficient, but not necessary condition to guarantee Equation~\ref{eqn:conslusion_singular_condition}. This is because Equation~\ref{eqn:conslusion_singular_condition} can be established with a weaker condition, 
\begin{equation*}
\lambda_{\min,o}^{'}\lambda_{\max,o}-\lambda_{\max,o}^{'}\lambda_{\min,o}+u_{o}^{2}(\lambda_{\min,o}^{'}-\lambda_{\min,o})\geq 0.
\label{eqn:weaker_assumption*}
\end{equation*}
We choose the stricter condition (Equation~\ref{eqn:assumption_singular_condition}) because it is conceptually easy to separate and eliminate the influence on the degree of observability from target's control input, $u_o$, which is unknown and uncontrolled.
\label{remark:rem2}
\end{rem}
\section*{Proof for Theorem~\ref{thm:lower_bound_ijo}}
\begin{proof}
The sensor-target relative state contribution of the local observability matrix $O_{i,j}(o,u_o)$ (Equation \ref{eqn:observability_i_j_o}) is,
\begin{equation}
O_{i,j}(o)=\begin{bmatrix}
    o_{x}-p_{ix} & o_{y}-p_{iy} \\ 
    o_{x}-p_{jx} & o_{y}-p_{jy} 
\end{bmatrix}.
\label{eqn:relative_state_i_j_o}
\end{equation}

The lower bound of the inverse of the condition number of $O_{i,j}(o,u_o)$ is $\underline{C}^{-1}(O_{i,j}(o,u_o))$. Our goal is to determine how to improve $\underline{C}^{-1}(O_{i,j}(o,u_o))$ by analyzing the sensor-target relative state contribution. This is similar to the work presented in~\cite{arrichiello2013observability}. 

We start by calculating the eigenvalues of $O_{i,j}^{T}(o)O_{i,j}(o)$.
\begin{equation}
O_{i,j}^{T}(o)O_{i,j}(o)=\begin{bmatrix}
    (o_{x}-p_{ix})^{2}, (o_{x}-p_{jx})(o_{y}-p_{iy}) \\ 
    (o_{y}-p_{iy})(o_{x}-p_{jx}), (o_{y}-p_{jy})^{2} 
\end{bmatrix}.
\label{eqn:symmetric_relative_ijo}
\end{equation}
Using polar coordinates (Equation \ref{eqn:polar_relative_sensor_target}), $O_{i,j}^{T}(o)O_{i,j}(o)$ can be rewritten as,
\begin{eqnarray}
\begin{bmatrix}
d_{io}^{2} \cos^{2}\theta_i$+$d_{jo}^{2} \cos^{2}\theta_j,d_{io}^{2}\sin\theta_i\cos\theta_i$+$d_{jo}^{2}\sin\theta_j\cos\theta_j\\ d_{jo}^{2}\sin\theta_j\cos\theta_j$+$d_{io}^{2}\sin\theta_i\cos\theta_i,d_{io}^{2} \sin^{2}\theta_i$+$d_{jo}^{2} \sin^{2}\theta_j
\end{bmatrix}. \nonumber
\label{eqn:oto_polar}
\end{eqnarray}
The eigenvalues are given by,
\begin{eqnarray}
&&\lambda_{\min}(O_{i,j}^{T}(o)O_{i,j}(o)) =\nonumber\\
&&\frac{d_{io}^{2}+d_{jo}^{2}-\sqrt{d_{io}^{4}+d_{jo}^{4}+2d_{io}^{2}d_{jo}^{2}\cos(2\theta_{ji})}}{2},\nonumber\\
&&\lambda_{\max}(O_{i,j}^{T}(o)O_{i,j}(o))=\nonumber\\
&&\frac{d_{io}^{2}+d_{jo}^{2}+\sqrt{d_{io}^{4}+d_{jo}^{4}+2d_{io}^{2}d_{jo}^{2}\cos(2\theta_{ji})}}{2}.
\label{eqnarray}
\end{eqnarray}
where $\theta_{ji}:=\theta_{j}-\theta_{i}$.
Then, by Equation ~\ref{eqn:inverse_lower_bound_simply} in the proof of Theorem~\ref{thm:lower_bound}, the claim is guaranteed. 
\label{proof:thm4}
\end{proof}
\section*{Proof for Theorem~\ref{thm:$ij_pair_extreme$}}
\begin{proof} First, for a fixed $\alpha$, 
$$\frac{\partial \underline{C}^{-1}(O_{i,j}(o,u_o))}{\partial \theta_{ji}}=0  \Rightarrow  \alpha^{2}\sin(2\theta_{ji})=0,$$
 $$\Rightarrow  \theta_{ji}=0, \pm \frac{\pi}{2}, \pi.
\label{eqn:partial_theta} $$
and 
$$\frac{\partial^{2}\underline{C}^{-1}(O_{i,j}(o,u_o))}{\partial \theta_{ji}^{2}}|_{\theta_{ji}=0,\pi} > 0. $$ 
$$\frac{\partial^{2}\underline{C}^{-1}(O_{i,j}(o,u_o))}{\partial \theta_{ji}^{2}}|_{\theta_{ji}=\pm \frac{\pi}{2}} < 0. $$ 
Thus, when $\alpha$ is fixed, $\underline{C}^{-1}(O_{i,j}(o,u_o))$ reaches its minimum, $\underline{C}^{-1}(O_{i,j}(o,u_o))|_{\theta_{ji}=0,~\pi}=0$, and reaches its maximum $\underline{C}^{-1}(O_{i,j}(o,u_o))|_{\theta_{ji}=\pm \frac{\pi}{2}}=\sqrt{\frac{1+\alpha^{2}-|1-\alpha^{2}|}{1+\alpha^{2}+|1-\alpha^{2}|+2u_{o}^{2}/d_{io}^{2}}}$.
Then, when $\theta_{ji}=\pm \frac{\pi}{2}$, 
$$\frac{\partial \underline{C}^{-1}(O_{i,j}(o,u_o))}{\partial \alpha}|_{\theta_{ji}=\pm \frac{\pi}{2}=0} $$
$$ \Rightarrow  \frac{\alpha(1-\alpha^{2})}{(1+\alpha^{2}+|1-\alpha^{2}|+2u_{o}^{2}/d_{io}^{2})^{2.5}}=0,$$
$$ \Rightarrow  \alpha=0,1, \alpha \to \infty.$$
and 
$$\frac{\partial^{2}\underline{C}^{-1}(O_{i,j}(o,u_o))}{\partial \alpha^{2}}|_{\theta_{ji}=\pm \frac{\pi}{2}, \alpha=0, \alpha \to \infty} > 0. $$ 
$$\frac{\partial^{2}\underline{C}^{-1}(O_{i,j}(o,u_o))}{\partial \alpha^{2}}|_{\theta_{ji}=\pm \frac{\pi}{2}, \alpha=1} < 0. $$ 
 Therefore, $\underline{C}^{-1}(O_{i,j}(o,u_o))|_{\theta_{ji}=\pm \frac{\pi}{2}}$ reaches its minimum $0$, at $\alpha=0$, or $\alpha \to \infty$, and reaches its maximum $\sqrt{\frac{1}{1+u_{o}^{2}/d_{io}^{2}}}$, at $\alpha=1$, as shown in Figure \ref{fig:lower_bound_surf_contour} where, $\underline{C}^{-1}(O_{i,j}(o,u_o))$ reaches its maximum at $\alpha=1$ and ${\theta_{ji}=\pm \frac{\pi}{2}}$. 

Note that, when $\theta_{ji}=0, \pi$,  $\underline{C}^{-1}(O_{i,j}(o,u_o))$ is always null and not influenced by $\alpha$. Besides, for a fixed $\theta_{ji} \neq 0, \pm \frac{\pi}{2}, \pi$, the minimum extreme point w.r.t. $\alpha$ still happens when $\alpha=0$ or $\alpha \to \infty$, but the maximum extreme point is not at $\alpha=1$, which is shown is the following derivations.
$$\frac{\partial \underline{C}^{-1}(O_{i,j}(o,u_o))}{\partial \alpha}|_{\theta_{ji} \neq 0, \pm \frac{\pi}{2}, \pi}=0$$
\begin{align*}
&\Rightarrow  \frac{\alpha[(1-\alpha^{2})(1-\cos(2\theta_{ji}))}{(1+\alpha^{2}+\sqrt{1+\alpha^{4}+2\alpha^{2}\cos(2\theta_{ji})}+2u_{o}^{2}/d_{io}^{2})^{2.5}}\\
&\frac{+u_{o}^{2}/d_{io}^{2}(\sqrt{1+\alpha^{4}+2\alpha^{2}\cos(2\theta_{ji})}-(\alpha^{2}+\cos(2\theta_{ji})))]}{(1+\alpha^{2}+\sqrt{1+\alpha^{4}+2\alpha^{2}\cos(2\theta_{ji})}+2u_{o}^{2}/d_{io}^{2})^{2.5}}\\
&=0.\\
&\Rightarrow \alpha=0, \alpha \to \infty.
\end{align*}
and $\Rightarrow$
\begin{align}
&(1-\alpha^{2})(1-\cos(2\theta_{ji}))+\nonumber\\
&\frac{u_{o}^{2}}{d_{io}^{2}}(\sqrt{1+\alpha^{4}+2\alpha^{2}\cos(2\theta_{ji})}-(\alpha^{2}+\cos(2\theta_{ji})))=0.
\label{eqn:maximum_unkown_extreme}
\end{align}
Since 
$$\frac{\partial^{2}\underline{C}^{-1}(O_{i,j}(o,u_o))}{\partial \alpha^{2}}|_{\theta_{ji} \neq 0, \pm \frac{\pi}{2}, \pi, \alpha=0, \alpha \to \infty} >  0,$$
$\frac{\partial \underline{C}^{-1}(O_{i,j}(o,u_o))}{\partial \alpha}|_{\theta_{ji} \neq 0, \pm \frac{\pi}{2}, \pi}$ also reaches its minimum at $\alpha=0$ or $\alpha \to \infty$. However, $\alpha=1$ cannot  make Equation \ref{eqn:maximum_unkown_extreme} established when $\theta_{ji} \neq 0, \pm \frac{\pi}{2}, \pi$, and thus, the extreme point is not at $\alpha=1$.
\end{proof}

\end{document}